\documentclass[a4paper]{article}
\usepackage{microtype}
\usepackage{graphicx}
\usepackage{subfigure}
\usepackage{booktabs} % for professional tables

\usepackage{amssymb}
\usepackage{graphicx}
\usepackage{amsmath}
\usepackage{xspace}
\usepackage{bm}
\usepackage{amsthm}
\usepackage{comment}
\usepackage{afterpage}
\usepackage{indentfirst}
\usepackage{enumerate}
\usepackage[numbers,sort&compress]{natbib}
\usepackage{color}
\usepackage{mathtools}
\usepackage{xr}

\newtheorem{theorem}{Theorem}
\newtheorem{lemma}[theorem]{Lemma}
\newtheorem{proposition}[theorem]{Proposition}
\newtheorem{corollary}[theorem]{Corollary}
\theoremstyle{definition}
\newtheorem{definition}{Definition}

\newtheorem{assumption}{Assumption}

\theoremstyle{definition}

\newtheorem{remark}{Remark}

\usepackage{hyperref}
\numberwithin{equation}{section}
\newcommand{\nbb}{\mathbb{N}}

\newcommand{\bw}{\mathbf{w}}
\newcommand{\ibb}{\mathbb{I}}
\newcommand{\xcal}{\mathcal{X}}
\newcommand{\ran}{\mathrm{Range}}
\newcommand{\prox}{\mathrm{Prox}}
\newcommand{\stab}{\mathrm{stab}}
\newcommand{\zcal}{\mathcal{Z}}
\newcommand{\ycal}{\mathcal{Y}}
\newcommand{\ebb}{\mathbb{E}}
\newcommand{\rbb}{\mathbb{R}}
\newcommand{\red}{} %\color{red}

\title{Fine-Grained Analysis of Stability and Generalization for Stochastic Gradient Descent\footnote{To appear in ICML 2020}}
\author{%
  Yunwen Lei$^{1}$\quad Yiming Ying$^{2}$\thanks{Corresponding author}\\[1.2pt]
  $^1$Department of Computer Science, University of Kaiserslautern,\\ 67653 Kaiserslautern, Germany\\[1.2pt]
  $^2$Department of Mathematics and Statistics, State University of New York\\ at Albany,
  12222 Albany, USA\\[1.2pt]
  \texttt{ylei@rhrk.uni-kl.de} \quad \texttt{yying@albany.edu}
}

\begin{document}

\maketitle
\linespread{1.2}

\begin{abstract}
%Algorithm stability is a popular concept to study the generalization behavior of stochastic gradient descent (SGD).
Recently there are a considerable amount of work devoted to the study of the algorithmic stability and  generalization for stochastic gradient descent (SGD). However, the existing stability analysis requires to impose restrictive assumptions on the boundedness of gradients, smoothness and convexity of loss functions. In this paper, we provide a fine-grained analysis of stability and generalization for SGD by substantially relaxing these assumptions. Firstly, we establish stability and generalization for SGD by removing the existing bounded gradient assumptions. The key idea is the introduction of a new stability measure called {\em on-average model stability}, for which we develop novel bounds controlled by the risks of SGD iterates. This yields  generalization bounds depending on the behavior of the best model, and leads to the {\em first-ever-known fast bounds} in the low-noise setting using stability approach. Secondly, the smoothness assumption is relaxed by considering loss functions with H\"older continuous (sub)gradients for which we show that optimal bounds are still achieved by balancing computation and stability. {\red To our best knowledge, this gives the {\em first-ever-known} stability and generalization bounds for SGD with {\em non-smooth} loss functions (e.g., hinge loss).}
Finally, we study learning problems with (strongly) convex objectives but non-convex loss functions.
\end{abstract}
\section{Introduction}
%SGD
%
%Importance of Generalization Error

%Three approaches to SGD: There exist problems that are learnable from
Stochastic gradient descent (SGD) has become the workhorse behind many machine learning problems. As an iterative algorithm, SGD updates the model sequentially upon receiving a new datum with a cheap per-iteration cost, making it amenable for big data analysis.  There is a plethora of theoretical work on its convergence analysis as an optimization algorithm~\citep[e.g.][]{duchi2011adaptive,zhang2004solving,rakhlin2012making,nemirovski2009robust,lacoste2012simpler,shamir2013stochastic}.

Concurrently, there are a considerable amount of work  with focus on  its generalization analysis  ~\citep{rosasco2015learning,hardt2016train,lin2016generalization,dieuleveut2016nonparametric,ying2016online}.  For instance,  using the tool of integral operator the work \citep{rosasco2015learning,lin2017optimal,ying2008online,dieuleveut2016nonparametric} studied the excess generalization error of SGD with the least squares loss, i.e.  the difference between the true risk of SGD iterates and the best possible risk. An advantage of this approach is its ability to capture the regularity of regression functions and the capacity of hypothesis spaces.  The results were further extended in \citet{lin2016generalization,lei2018stochastic}  based on tools of  empirical processes which are able to deal with general convex functions even without a smoothness assumption. The idea is to bound the complexity of SGD iterates in a controllable manner, and apply concentration inequalities in empirical processes to control the uniform deviation between population risks and empirical risks over a ball to which the SGD iterates belong.

Recently, in the seminal work \cite{hardt2016train} the authors studied the generalization bounds of SGD via algorithmic stability \cite{bousquet2002stability,elisseeff2005stability} for convex, strongly convex and non-convex problems. This motivates several appealing work on some weaker stability measures of SGD that still suffice for guaranteeing generalization~\citep{kuzborskij2018data,zhou2018generalization,charles2018stability}. An  advantage of this stability approach is that it considers only the particular model produced by the algorithm, and can imply generalization bounds independent of the dimensionality.

However, the existing stability analysis of SGD is established under the strong assumptions on the loss function such as the boundedness of the gradient and strong smoothness.  Such assumptions are very restrictive which are not satisfied in many standard contexts.
{\red For example, the bounded gradient assumption does not hold for the simple least-squares regression, where the model parameter belongs to an unbounded domain.
The strong smoothness assumption does not hold for the popular support vector machine.}
Furthermore, the analysis in the strongly convex case requires strong convexity of each loss function which is not true for many problems such as  the important problem of least squares regression.
%Furthermore, the analysis in the (strongly) convex case requires convexity of each loss function which is not true for many problems such as  the important problem of AUC maximization \citep{ying2016stochastic,zhao2011online,natole2018stochastic}.

In this paper, we provide a fine-grained analysis of stability and generalization for SGD. {\red Our new results remove the bounded gradient assumption for differentiable loss functions and remove the strong smoothness assumptions for Lipschitz continuous loss functions, and therefore broaden the impact of the algorithmic stability approach for generalization analysis of SGD}.  In summary, our main contributions are listed as follows.

 \noindent $\bullet$~ Firstly, we study stability and generalization for SGD by removing the existing bounded gradient assumptions. The key is an introduction of a novel stability measure called on-average model stability, whose connection to generalization is established by using the smoothness of loss functions able to capture the low risks of output models for better generalization. An advantage of on-average model stability is that the corresponding bounds involve a weighted sum of empirical risks instead of the uniform Lipschitz constant. The weighted sum of empirical risks can be bounded via tools in analyzing optimization errors, which implies a key message that optimization is beneficial to generalization.
 Furthermore, our stability analysis allows us to develop generalization bounds depending on the risk of the best model. In particular, we have established fast generalization bounds $O(1/n)$ for the setting of low noises, where $n$ is the sample size. To our best knowledge, this is the first fast generalization bound of SGD based on stability approach in a low-noise setting.

\noindent $\bullet$~  Secondly, we consider loss functions with their (sub)gradients satisfying the H\"older continuity which is a much weaker condition than the strong smoothness in the literature. Although stability decreases by weakening the smoothness assumption, optimal generalization bounds can be surprisingly achieved by balancing computation and stability.
{\red In particular, we show that optimal generalization bounds can be achieved for the hinge loss by running SGD with $O(n^2)$ iterations}.
Fast learning rates are further derived in the low-noise case.

\noindent $\bullet$~ Thirdly, we study learning problems with (strongly) convex objectives but non-convex individual loss functions. The nonconvexity of loss functions makes the corresponding gradient update no longer non-expansive, and therefore the arguments in \citet{hardt2016train} do not apply. We bypass this obstacle by developing a novel quadratic inequality of the stability using only the convexity of the objective, which shows that this relaxation affects neither generalization nor computation. %We apply our results to two specific problems with either non-convex or non-strongly convex loss functions.%, to which the existing results do not apply.

%Then,  Surprisingly, , where the .  An extension of the discussions to SGD without replacement is also given.

The paper is structured as follows. We discuss the related work in Section \ref{sec:work} and formulate the problem in Section \ref{sec:formulation}. The  stability and generalization for learning with convex loss functions is presented in Section \ref{sec:no-bounded-gradient}.
In Sections  \ref{sec:convex} and \ref{sec:strong}, we consider problems with relaxed convexity and relaxed strong convexity, respectively.   We conclude the paper in  Section \ref{sec:conclusion}.

\section{Related Work\label{sec:work}}%devroye1979distribution , structured prediction~\citep{london2013collective} maurer2005algorithmic
In this section, we discuss related work on algorithmic stability, stability of stochastic optimization algorithms and generalization error of SGD.

\smallskip
\textbf{Algorithmic Stability}. The study of  stability can be dated back to \citet{rogers1978finite}. A modern framework of quantifying generalization via stability was established in the  paper~\citep{bousquet2002stability}, where a concept of uniform stability was introduced and studied for empirical risk minimization (ERM) in the strongly convex setting.
This framework was then extended to study randomized learning algorithms~\citep{elisseeff2005stability}, transfer learning~\citep{kuzborskij2018data} and privacy-preserving learning~\citep{dwork2018privacy}, etc. The interplay between various notions of stability, learnability and consistency was further studied \citep{shalev2010learnability,rakhlin2005stability}.
The power of stability analysis is especially reflected by its ability in deriving optimal generalization bounds in expectation~\citep{shalev2010learnability}.
Very recently, almost optimal high-probability generalization bounds were established via the stability approach~\citep{feldman2018generalization,feldman2019high,bousquet2019sharper}. In addition to the notion of uniform stability mentioned above, various other notions of   stability were recently introduced, including uniform argument stability~\citep{liu2017algorithmic} and hypothesis set stability~\citep{foster2019hypothesis}.

\smallskip
\textbf{Stability of Stochastic Optimization Algorithms}. In the seminal paper~\citep{hardt2016train}, the co-coercivity of gradients was used to study the uniform stability of SGD in convex, strongly convex and non-convex  problems.
The uniform stability was relaxed to a weaker notion of on-average stability~\citep{shalev2010learnability}, for which the corresponding bounds of SGD can capture the impact of the risk at the initial point~\citep{kuzborskij2018data} and the variance of stochastic gradients~\citep{zhou2018generalization}. For non-convex learning problems satisfying either a gradient dominance or a quadratic growth condition, pointwise-hypothesis stabilities were studied for a class of learning algorithms that converge to global optima~\citep{charles2018stability}, which relaxes and extends the uniform stability of ERM under strongly convex objectives~\citep{bousquet2002stability}.
A fundamental stability and convergence trade-off of iterative optimization algorithms was recently established, where it was shown that a faster converging algorithm can not be too stable, and vice versa~\citep{chen2018stability}. This together with some uniform stability bounds for several first-order algorithms established there, immediately implies new convergence lower bounds for the corresponding algorithms. Algorithmic stability was also established for stochastic gradient Langevin dynamics with non-convex objectives~\citep{mou2018generalization,li2019generalization} and SGD implemented in a stagewise manner~\citep{yuan2019stagewise}.%, which is a variant of SGD with additive Gaussian noises.

\smallskip
\textbf{Generalization Analysis of SGD}.
A framework to study the generalization performance of large-scale stochastic optimization algorithms was established in \citet{bousquet2008tradeoffs}, where three factors influencing generalization behavior were identified as optimization errors, estimation errors and approximation errors. Uniform stability was used to establish generalization bounds $O(1/\sqrt{n})$ in expectation for SGD for convex and strongly  smooth cases~\citep{hardt2016train}. For convex and nonsmooth learning problems, generalization bounds $O(n^{-\frac{1}{3}})$ were established based on the uniform convergence principle~\citep{lin2016generalization}. An interesting observation is that an implicit regularization can be achieved without an explicit regularizer by tuning either the number of passes or the step sizes~\citep{rosasco2015learning,lin2016generalization}. For the specific least squares loss, optimal excess  generalization error bounds (up to a logarithmic factor) were established for SGD based on the integral operator approach~\citep{lin2017optimal,pillaud2018statistical}.
The above mentioned generalization results are in the form of expectation.  High-probability bounds were  established based on either an uniform-convergence approach~\citep{lei2018stochastic} or an algorithmic stability approach~\citep{feldman2019high}.
A novel combination of PAC-Bayes and algorithmic stability was used to study the generalization behavior of SGD, a promising property of which is its applications to all posterior distributions of algorithms' random hyperparameters~\citep{london2017pac}. %generalization bounds~\citep{zhang2019stochastic} or
%Fast decay of  test errors~\citep{pillaud2018exponential} are possible for SGD under some additional assumptions, e.g., low noise conditions.
%The above mentioned work consider multi-pass SGD.
%In the literature, the generalization performance of one-pass SGD has been well studied %~\citep{ying2008online,nemirovski2009robust,rakhlin2012making,tarres2014online,orabona2014simultaneous,ying2006online}, where each training example can be only used once.
\section{Problem Formulation\label{sec:formulation}}
Let $S=\{z_1,\ldots,z_n\}$ be a set of training examples independently drawn from a probability measure $\rho$ defined over a sample space $\zcal=\xcal\times\ycal$, where $\xcal\subseteq\rbb^d$ is an input space and $\ycal\subseteq\rbb$ is an output space. Our aim is to learn a prediction function parameterized by $\bw\in\Omega\subseteq\rbb^d$ to approximate the relationship between an input variable $x$ and an output variable $y$. We quantify the loss of a model $\bw$ on an example $z=(x,y)$ by $f(\bw;z)$. The corresponding  empirical and population risks are respectively given by
\[
F_S(\bw)=\frac{1}{n}\sum_{i=1}^{n}f(\bw;z_i)\quad\text{and}\quad F(\bw)=\ebb_z[f(\bw;z)].
\]
Here we use $\ebb_z[\cdot]$ to denote the expectation with respect to (w.r.t.) $z$.
%The generalization gap of a model $\bw$ is defined by $\epsilon_{\text{gap}}=F(\bw)-F_S(\bw)$.
In this paper, we consider stochastic learning algorithms $A$, and denote by $A(S)$ the model produced by running $A$ over the training examples $S$. %In this case, we have two sources of randomness: one randomness due to the sampling of training examples and one randomness due to the adopted algorithm.

We are interested in studying the {\em excess generalization error} $F(A(S))-F(\bw^*)$, where $\bw^*\in\arg\min_{\bw\in\Omega}F(\bw)$ is the one with the best prediction performance over $\Omega$. It can be decomposed as
\begin{align}
\ebb_{S,A}\big[F(A(S))\!-&\!F(\bw^*)\big]\!=\!\ebb_{S,A}\big[F(A(S))\!-\!F_S(A(S))\big]\notag\\
&+\ebb_{S,A}\big[F_S(A(S))-F_S(\bw^*)\big].\label{decomposition}
\end{align}
The first term is called the estimation error due to the approximation of the unknown probability measure $\rho$ based on sampling. The second term is called the optimization error induced by running an optimization algorithm to minimize the empirical objective, which can be addressed by tools in optimization theory. A popular approach to control estimation errors is to consider the stability of the algorithm, for which a widely used
stability measure is the uniform stability~\citep{elisseeff2005stability,hardt2016train}.
\begin{definition}[Uniform Stability\label{def:unif-stab}]
  A stochastic algorithm $A$ is $\epsilon$-uniformly stable if for all training datasets $S,\widetilde{S}\in\zcal^n$ that differ by at most one example, we have
  \begin{equation}\label{unif-stab}
  \sup_z\ebb_A\big[f(A(S);z)-f(A(\widetilde{S});z)\big]\leq\epsilon.
  \end{equation}
\end{definition}

The celebrated relationship between generalization and uniform stability is established in the following lemma~\citep{shalev2010learnability,hardt2016train}.
\begin{lemma}[Generalization via uniform stability\label{lem:gen-stab}]
  Let $A$ be $\epsilon$-uniformly stable. Then
  \[
  \big|\ebb_{S,A}\big[F_S(A(S))-F(A(S))\big]\big|\leq\epsilon.
  \]
\end{lemma}

Throughout the paper, we restrict our interest to  a specific algorithm called projected stochastic gradient descent.
It is worth mentioning that our main results in Section \ref{sec:no-bounded-gradient} hold also when $\Omega=\rbb^d$, i.e., no projections.
\begin{definition}[Projected Stochastic Gradient Descent]
Let $\Omega\subseteq\rbb^d$ and $\Pi_{\Omega}$ denote the projection on $\Omega$.
Let $\bw_1=0\in\rbb^d$ be an initial point and $\{\eta_t\}_t$ be a sequence of positive step sizes. Projected SGD updates models by
\begin{equation}\label{SGD}
  \bw_{t+1}=\Pi_{\Omega}\big(\bw_t-\eta_t\partial f(\bw_t;z_{i_t})\big),
\end{equation}
where $\partial f(\bw_t,z_{i_t})$ denotes a subgradient of $f$ w.r.t. the first argument and $i_t$ is independently drawn from the uniform distribution over $\{1,\ldots,n\}$.
%We denote by $I=\{i_1,i_2,\ldots\}$ the set of randomly selected indices.
\end{definition}
%We denote $\ebb_A$ the conditional expectation w.r.t. $I$.
%Let $S=\{z_1,\ldots,z_n\}$ and $\widetilde{S}=\{\tilde{z}_1,\ldots,\tilde{z}_n\}$ be two set of training examples that differ by a single example.

Note if $f$ is differentiable, then $\partial f$ denotes the gradient of $f$ w.r.t. the first argument.
%We now introduce some necessary concepts.
We say a function $g:\rbb^d\mapsto\rbb$ is $\sigma$-strongly convex if
\begin{equation}\label{strong-convex}
g(\bw)\geq g(\tilde{\bw})+\langle\bw-\tilde{\bw},\partial g(\tilde{\bw})\rangle+\frac{\sigma}{2}\|\bw-\tilde{\bw}\|_2^2
\end{equation}
for all $\bw,\tilde{\bw}\in\rbb^d$,
where $\langle\cdot,\cdot\rangle$ denotes the inner product and $\|\bw\|_2$ denotes the $\ell_2$ norm of $\bw=(w_1,\ldots,w_d)$, i.e., $\|\bw\|_2=\big(\sum_{j=1}^{d}w_j^2\big)^{\frac{1}{2}}$.
If \eqref{strong-convex} holds with $\sigma=0$, then we say $g$ is convex.
We denote $B\asymp \widetilde{B}$ if there are absolute constants $c_1$ and $c_2$ such that $c_1B\leq \widetilde{B}\leq c_2B$.
%We say a differentiable function $g$ is $L$-smooth if $\big\|\nabla g(\bw)-\nabla g(\tilde{\bw})\big\|_2\leq L\|\bw-\tilde{\bw}\|_2$ for all $\bw,\tilde{\bw}\in\rbb^d$.

%We always assume $f(\bw;z)\geq0$ for all $\bw$ and $z$.
%Several results require a Lipschitz continuity assumption.

%\begin{equation}\label{monotonicity}
%\langle\bw-\tilde{\bw},\nabla g(\bw)-\nabla g(\tilde{\bw})\rangle\geq0,\quad\forall\bw,\tilde{\bw}\in\rbb^d.
%\end{equation}
%In the literature, a mapping $\bw\mapsto\nabla g(\bw)$ satisfying \eqref{monotonicity} is said to be monotone.

\section{\red Stability with Convexity\label{sec:no-bounded-gradient}}
An essential assumption to establish the uniform stability of SGD is the uniform Lipschitz continuity (boundedness of gradients) of loss functions as follows~\citep{hardt2016train,zhou2018generalization,bousquet2002stability,kuzborskij2018data,charles2018stability}.
%This boundedness assumption may not hold if $\Omega$ is unbounded, for which the existing stability bounds do not apply.
\begin{assumption}\label{ass:lipschitz}%[Lipschitzness]
  We assume $\|\partial f(\bw;z)\|_2\leq G$ for all $\bw\in\Omega$ and $z\in\zcal$. %, i.e., the Lipschitz continuity of $f(\cdot,z)$ for all $z\in\zcal$.
\end{assumption}
Unfortunately, the Lipschitz constant $G$ can be very large or even infinite for some learning problems. Consider the simple least squares loss $f(\bw;z)=\frac{1}{2}(\langle\bw,x\rangle-y)^2$ with the gradient $\partial f(\bw;z)=(\langle\bw,x\rangle-y)x$. In this case the $G$-Lipschitzness of $f$ requires to set $G=\sup_{\bw\in\Omega}\sup_{z\in\zcal}\|(\langle\bw,x\rangle-y)x\|_2$, which is  infinite if $\Omega$ is unbounded.
As another example, the Lipschitz constant of deep neural networks can be prohibitively large.
In this case, existing stability bounds fail to yield meaningful generalization bounds.  Furthermore, another critical assumption in the literature is the  $L$-smoothness on $f$, i.e. for any $z$ and $\bw, \tilde{\bw}\in\rbb^d$
  \begin{equation}\label{Lip-condition}
    \big\|\partial f(\bw,z)-\partial f(\tilde{\bw},z)\big\|_2\leq L\|\bw-\tilde{\bw}\|_2.
  \end{equation}

In this section, we will remove the boundedness assumption on the  gradients for differentiable loss functions, and  establish  stability and generalization only under the  assumption where  loss functions have H\"older continuous (sub)gradients--a condition much weaker than the strong smoothness~\citep{nesterov2015universal,lei2018convergence,ying2017unregularized}.
{\red Note that the loss functions can be \emph{non-differentiable} if $\alpha=0$.}
\begin{definition}\label{def:holder}
  Let $L>0,\alpha\in[0,1]$. We say $\partial f$ is $(\alpha,L)$-H\"older continuous if for all $\bw,\tilde{\bw}\in\rbb^d$ and $z\in \mathcal{Z}$,
  \begin{equation}\label{holder-condition}
    \big\|\partial f(\bw,z)-\partial f(\tilde{\bw},z)\big\|_2\leq L\|\bw-\tilde{\bw}\|_2^\alpha.
  \end{equation}\end{definition}
If \eqref{holder-condition} holds with $\alpha=1$, then $f$ is smooth as defined by \eqref{Lip-condition}.
{\red If \eqref{holder-condition} holds with $\alpha=0$, then this amounts to saying that $f$ is Lipschitz continuous as considered in Assumption \ref{ass:lipschitz}.}
Examples of loss functions satisfying Definition \ref{def:holder} include the $q$-norm hinge loss $f(\bw;z)=\big(\max(0,1\!-\!y\langle\bw,x\rangle)\big)^q$ for classification and the $q$-th power absolute distance loss $f(\bw;z)=|y\!-\!\langle\bw,x\rangle|^q$ for regression~\citep{steinwart2008support}, whose (sub)gradients are $(q\!-\!1,C)$-H\"older continuous for some $C>0$ if $q\in[1,2]$. {\red If $q=1$, we get the hinge loss and absolute distance loss with wide applications in machine learning and statistics.}

\subsection{On-average model stability}
{\red The key to remove the bounded gradient assumption and the strong smoothness assumption} is the introduction of a novel stability measure which we refer to as the on-average model stability. We use the term ``on-average model stability" to differentiate it from on-average stability in  \citet{kearns1999algorithmic,shalev2010learnability} as  we measure stability on \emph{model parameters} $\bw$ instead of \emph{function values}. Intuitively, on-average model stability measures the on-average sensitivity of models by traversing the perturbation of each single coordinate.
\begin{definition}[On-average Model Stability\label{def:aver-stab}]
  Let $S=\{z_1,\ldots,z_n\}$ and $\widetilde{S}=\{\tilde{z}_1,\ldots,\tilde{z}_n\}$ be drawn independently from $\rho$. For any $i=1,\ldots,n$, define
  $S^{(i)}=\{z_1,\ldots,z_{i-1},\tilde{z}_i,z_{i+1},\ldots,z_n\}$ as the set formed from $S$ by replacing the $i$-th element with $\tilde{z}_i$.
  We say a randomized algorithm $A$ is $\ell_1$ on-average model $\epsilon$-stable if
  \[
  \ebb_{S,\widetilde{S},A}\Big[\frac{1}{n}\sum_{i=1}^{n}\|A(S)-A(S^{(i)})\|_2\Big]\leq\epsilon,
  \]
  and $\ell_2$ on-average model $\epsilon$-stable if
  \[
  \ebb_{S,\widetilde{S},A}\Big[\frac{1}{n}\sum_{i=1}^{n}\|A(S)-A(S^{(i)})\|_2^2\Big]\leq\epsilon^2.
  \]
\end{definition}

%Although Part (b) was not explicitly stated in the literature, it holds clearly from the arguments in \citet{hardt2016train,shalev2010learnability}.
%Part (c) is new to our knowledge, which is proved in Section \ref{sec:proof-gen-stab}.
%

In the following theorem, we build the connection between generalization in expectation and the on-average model stabilities to be proved in Appendix \ref{sec:proof-gen-stab}. Although the generalization by $\ell_1$ on-average model stability requires Assumption \ref{ass:lipschitz}, it is removed for $\ell_2$ on-average model stability. We introduce a free parameter $\gamma$ to tune according to the property of problems.
Note we require a convexity assumption in Part (c) by considering non-smooth loss functions.
{\red Let $c_{\alpha,1}=(1+1/\alpha)^{\frac{\alpha}{1+\alpha}}L^{\frac{1}{1+\alpha}}$ if $\alpha>0$ and $c_{\alpha,1}=\sup_z\|\partial f(0;z)\|_2+L$ if $\alpha=0$.}
%Note we require a convexity assumption in Part (c) by considering non-smooth loss functions.
\begin{theorem}[Generalization via Model Stability\label{thm:gen-model-stab}]
Let $S,\widetilde{S}$ and $S^{(i)}$ be constructed as Definition \ref{def:aver-stab}. Let $\gamma>0$.
\begin{enumerate}[(a)]
  \item Let $A$ be $\ell_1$ on-average model $\epsilon$-stable and Assumption \ref{ass:lipschitz} hold. Then
    \[
  \big|\ebb_{S,A}\big[F_S(A(S))-F(A(S))\big]\big|\leq G\epsilon.
  \]
  \item If for any $z$, the function $\bw\mapsto f(\bw;z)$ is nonnegative and $L$-smooth, then
  \[
    \ebb_{S,A}\big[F(A(S))-F_S(A(S))\big]  \leq \frac{L}{\gamma}\ebb_{S,A}\big[F_S(A(S))\big]
    +\frac{L+\gamma}{2n}\sum_{i=1}^{n}\ebb_{S,\widetilde{S},A}\big[\|A(S^{(i)})-A(S)\|_2^2\big].
  \]
  \item If for any $z$, the function $\bw\mapsto f(\bw;z)$ is nonnegative, convex and $\bw\mapsto \partial f(\bw;z)$ is $(\alpha,L)$-H\"older continuous with $\alpha\in[0,1]$, then
    \[
    \ebb_{S,A}\big[F(A(S))-F_S(A(S))\big]\leq\frac{c^2_{\alpha,1}}{2\gamma}\ebb_{S,A}\Big[F^{\frac{2\alpha}{1+\alpha}}(A(S))\Big]
    +\frac{\gamma}{2n}\sum_{i=1}^{n}\ebb_{S,\widetilde{S},A}\big[\|A(S^{(i)})-A(S)\|_2^2\big].
  \]
\end{enumerate}
\end{theorem}
\begin{remark}
We explain here the benefit of $\ell_2$ on-average model stability.
If $A$ is $\ell_2$ on-average model $\epsilon$-stable, then we take $\gamma=\sqrt{2L\ebb\big[F_S(A(S))\big]}/\epsilon$ in Part (b) and derive
\[
\ebb\big[F(A(S))-F_S(A(S))\big]\leq L\epsilon^2/2+\sqrt{2L\ebb\big[F_S(A(S))\big]}\epsilon.
\]
In particular, if the output model has a small empirical risk in the sense of $\ebb\big[F_S(A(S))\big]=O(1/n)$, we derive $\ebb\big[F(A(S))-F_S(A(S))\big]=O(\epsilon^2+\epsilon/\sqrt{n})$. That is, our relationship between the generalization and $\ell_2$ on-average stability allows us to exploit small risk of output model to get a generalization bound with an improved dependency on the stability measure $\epsilon$. As a comparison, the discussions based on uniform stability (Lemma \ref{lem:gen-stab}) and the $\ell_1$ on-average model stability (Part (a)) only show $\ebb\big[F(A(S))-F_S(A(S))\big]=O(\epsilon)$, which fail to exploit the low-noise condition.
We can also take $\gamma=c_{\alpha,1}\big(\ebb[F(A(S))]\big)^{\frac{\alpha}{1+\alpha}}/\epsilon$ in part (c) to derive
\[
\ebb\big[F(A(S))-F_S(A(S))\big]=O\Big(\epsilon\big(\ebb[F(A(S))]\big)^{\frac{\alpha}{1+\alpha}}\Big).
\]
The above equation can be written as an inequality of $\ebb[F(A(S))-F_S(A(S))]$ (using the sub-additivity of $t\mapsto t^{\frac{\alpha}{1+\alpha}}$), from which we derive
\[
\ebb\big[F(A(S))-F_S(A(S))\big]\!=\!O\Big(\epsilon^{1+\alpha}\!+\!\epsilon\big(\ebb[F_S(A(S))]\big)^{\frac{\alpha}{1+\alpha}}\Big).
\]
If $\ebb[F_S(A(S))]$ is small, this also implies an improved dependency of the generalization bound on $\epsilon$.
\end{remark}

\subsection{Strongly smooth case}
To justify the effectiveness of the on-average model stability, we first consider its application to learning with smooth loss functions.
We first study stability and then generalization.

\textbf{Stability bounds}.
The following theorem to be proved in Appendix \ref{sec:proof-lip-stab} establishes on-average model stability bounds in the smooth setting.
A key difference from the existing stability bounds is that the uniform Lipschitz constant $G$ is replaced by empirical risks. Since we are minimizing empirical risks by SGD, it is expected that these risks would be significantly smaller than the uniform Lipschitz constant.
Actually we will control the weighted sum of empirical risks by tools in analyzing optimization errors.
In the optimistic case with $F(\bw^*)=0$, we expect $\ebb_{S,\widetilde{S},A}[F_S(\bw_t)]=O(1/t)$, and in this case the discussion based on on-average model stability would imply significantly better generalization bounds.
The idea of introducing a parameter $p$ in \eqref{on-average-l2} is to make $(1+p/n)^t\leq e$ by setting $p=n/t$, where $e$ is the base of the nature logarithm.
\begin{theorem}[Stability bounds]\label{thm:on-average}
  Assume for all $z\in\zcal$, the map $\bw\mapsto f(\bw;z)$ is nonnegative, convex and $L$-smooth.
  Let $S,\widetilde{S}$ and $S^{(i)}$ be constructed as Definition \ref{def:aver-stab}. Let $\{\bw_t\}$ and $\{\bw_t^{(i)}\}$ be produced by \eqref{SGD} with $\eta_t\leq2/L$ based on $S$ and $S^{(i)}$, respectively. Then for any $p>0$ we have
  \begin{equation}\label{on-average}
  \ebb_{S,\widetilde{S},A}\Big[\frac{1}{n}\sum_{i=1}^{n}\|\bw_{t+1}-\bw_{t+1}^{(i)}\|_2\Big]\leq \frac{2\sqrt{2L}}{n}\sum_{j=1}^{t}\eta_j\ebb_{S,A}\Big[\sqrt{F_S(\bw_j)}\Big].
  \end{equation}
  and
  \begin{equation}\label{on-average-l2}
    \ebb_{S,\widetilde{S},A}\Big[\frac{1}{n}\sum_{i=1}^{n}\|\bw_{t+1}-\bw_{t+1}^{(i)}\|_2^2\Big]
    \leq  \frac{8(1+p^{-1})L}{n}\sum_{j=1}^{t}(1+p/n)^{t-j}\eta_j^2\ebb_{S,A}\big[F_S(\bw_j)\big].
  \end{equation}
\end{theorem}

\begin{remark}\label{rem:kuzborskij}
\citet{kuzborskij2018data} developed an interesting on-average stability bound $O(\frac{\tilde{\sigma}}{n}\sum_{j=1}^{t}\eta_j)$ under the bounded variance assumption $\ebb_{S,z}\big[\|\partial f(\bw_t;z)-\partial F(\bw_t;z)\|_2^2\big]\leq\tilde{\sigma}^2$ for all $t$. Although this  bound successfully replaces the uniform Lipschitz constant by the milder uniform variance constant $\tilde{\sigma}$, the corresponding generalization analysis still requires a bounded gradient assumption.
A nice property of the stability bound in \citet{kuzborskij2018data} is that it depends on the quality of the initialization, i.e., the stability increases if we start with a good model. Our stability bound also enjoys this property. As we can see from Theorem \ref{thm:on-average}, the stability increases if we find good models with small optimization errors in the optimization process. This illustrates a key message that optimization is beneficial to improve the generalization.
\end{remark}

\begin{remark}
  The stability bounds in Theorem \ref{thm:on-average} can be extended to the non-convex case. Specifically, let assumptions of Theorem \ref{thm:on-average}, except the convexity of $\bw\mapsto f(\bw;z)$, hold. Then for any $p>0$ one gets (see Proposition~\ref{thm:on-average-non-convex})
  \begin{multline*}
    \ebb_{S,\widetilde{S},A}\Big[\frac{1}{n}\sum_{i=1}^{n}\|\bw_{t+1}-\bw_{t+1}^{(i)}\|_2^2\Big]
    \leq
    (1+p/n)(1+\eta_tL)^2
    \ebb_{S,\widetilde{S},A}\Big[\frac{1}{n}\sum_{i=1}^{n}\|\bw_{t}-\bw_{t}^{(i)}\|_2^2\Big]\\
    + \frac{8(1+p^{-1})L\eta_t^2}{n}\ebb_{S,A}\big[F_S(\bw_t)\big].
  \end{multline*}
  This result improves the recurrence relationship in \citet{hardt2016train} for uniform stability by replacing the uniform Lipschitz constant with empirical risks.
\end{remark}

\textbf{Generalization bounds}.
We now establish generalization bounds based on $\ell_2$ on-average model stability. This approach not only removes a bounded gradient assumption, but also allows us to fully exploit the smoothness of loss functions to derive bounds depending on the behavior of the best model $\bw^*$.
As we will see in Corollary \ref{cor:gen-lipschitz}, Theorem \ref{thm:gen-lipschitz} interpolates between $O(1/\sqrt{n})$ bound in the ``pessimistic'' case ($F(\bw^*)>0$) and the $O(1/n)$ bound in the ``low-noise'' case ($F(\bw^*)=0$)~\citep{srebro2010smoothness,DBLP:journals/corr/abs-2002-09769},
which is becoming more and more interesting in the deep learning era with possibly more parameters than training examples.
To our best knowledge, this is the first optimistic bound for SGD based on a stability approach.
Eq. \eqref{cor-gen-lipschitz-b} still holds if $F(\bw^*)=O(1/n)$.
The proofs are given in Appendix~\ref{sec:proof-lip-gen}.
\begin{theorem}[Generalization bounds]\label{thm:gen-lipschitz}
  Assume for all $z\in\zcal$, the function $\bw\mapsto f(\bw;z)$ is nonnegative, convex and $L$-smooth.
  Let $\{\bw_t\}$ be produced by \eqref{SGD} with nonincreasing step sizes satisfying $\eta_t\leq1/(2L)$. If $\gamma\geq1$, then
\[
\ebb_{S,A}[F(\bw_T^{(1)})]-F(\bw^*)=O\Big(\Big(\frac{1}{\gamma}+\frac{\sum_{t=1}^{T}\eta_t^2}{\sum_{t=1}^{T}\eta_t}\Big)F(\bw^*)
+\frac{1}{\sum_{t=1}^{T}\eta_t}+
\frac{\gamma(1+T/n)}{n}\big(1+\sum_{t=1}^{T}\eta_t^2F(\bw^*)\big)\Big),
\]
where $\bw_T^{(1)}=\big(\sum_{t=1}^{T}\eta_t\bw_t\big)/\sum_{t=1}^{T}\eta_t$.
\end{theorem}

\begin{corollary}\label{cor:gen-lipschitz}
  Assume for all $z\in\zcal$, the function $\bw\mapsto f(\bw;z)$ is nonnegative, convex and $L$-smooth.
  \begin{enumerate}[(a)]
    \item Let $\{\bw_t\}$ be produced by \eqref{SGD} with $\eta_t=c/\sqrt{T}\leq1/(2L)$ for a constant $c>0$.
  If $T\asymp n$, then
  \begin{equation}\label{cor-gen-lipschitz-a}
  \ebb_{S,A}[F(\bw_T^{(1)})]-F(\bw^*)=O\Big(\frac{F(\bw^*)+1}{\sqrt{n}}\Big).
  \end{equation}
    \item Let $\{\bw_t\}$ be produced by \eqref{SGD} with $\eta_t=\eta_1\leq1/(2L)$. If $F(\bw^*)=0$
    and $T\asymp n$, then
    \begin{equation}\label{cor-gen-lipschitz-b}
  \ebb_{S,A}[F(\bw_T^{(1)})]-F(\bw^*)=O(1/n).
  \end{equation}
  \end{enumerate}
\end{corollary}

%\begin{remark}
%  If $T\asymp n$, then $(1+T/n)(L+\sqrt{n})=O(\sqrt{n})$. Therefore, the assumption on $n$ in Part (a) reduces to $\sqrt{n}\geq \tilde{c}$ for a constant $\tilde{c}$, which is very mild. If $T\asymp n$, then $T+T^2/n=O(n)$. Therefore, our assumption on $n$ in Part (b) reduces to $\eta_1^2=O(\tilde{c})$ for a constant $\tilde{c}$.
%\end{remark}

\begin{remark}
  Based on the stability bound in \citet{hardt2016train}, we can show $\ebb_{S,A}\big[F(\bw_T^{(1)})\big]-F(\bw^*)$ decays as
  \begin{equation}\label{gen-hardt}
    \frac{2G^2\sum_{t=1}^{T}\eta_t}{n}+O\Big(\frac{\sum_{t=1}^{T}\eta_t^2F(\bw^*)+1}{\sum_{t=1}^{T}\eta_t}\Big),
  \end{equation}
  from which one can derive the $O(1/\sqrt{n})$ bound at best even if $F(\bw^*)=0$. The improvement of our bounds over \eqref{gen-hardt} is due to the consideration of on-average model stability bounds involving empirical risks (we use the same optimization error bounds in these two approaches).
  Based on the on-average stability bound in \citet{kuzborskij2018data}, one can derive a generalization bound similar to \eqref{gen-hardt} with $G^2$ replaced by $G\tilde{\sigma}$ ($\tilde{\sigma}$ is the uniform variance constant in Remark \ref{rem:kuzborskij}), which also could not yield a fast bound $O(1/n)$ if $F(\bw^*)=0$.
\end{remark}

\begin{remark}
  We compare here our results with some fast bounds for SGD.
  Some fast convergence rates of SGD were recently derived for SGD under low noise conditions~\citep{ma2018power,bassily2018exponential,srebro2010smoothness} or growth conditions relating stochastic gradients to full gradients~\citep{vaswani2019fast}. The discussions there mainly focused on optimization errors, which are measured w.r.t. the iteration number $t$. As a comparison, our fast rates measured by $n$ are developed for generalization errors of SGD (Part (b) of Corollary \ref{cor:gen-lipschitz}), for which we need to trade-off optimization errors and estimation errors by stopping at an appropriate iteration number. Fast generalization bounds are also established for the specific least squares based on an integral operator approach~\citep{lin2017optimal,dieuleveut2017harder,mucke2019beating,pillaud2018statistical}. However, these discussions heavily depend on the structure of the square loss and require capacity assumptions in terms of the decay rate of eigenvalues for the associated integral operator. As a comparison, we consider general loss functions and do not impose a capacity assumption.
  %Our main contribution is on the part of estimation errors.%Actually, there are no discussions on estimation generalization from empirical risk to population risks.
\end{remark}

\subsection{Non-smooth case}

As a further application, we apply our on-average model stability to learning with non-smooth loss functions, which have not been studied in the literature.

\medskip

\textbf{Stability bounds}.
The following theorem to be proved in Appendix \ref{sec:proof-holder-stab} establishes stability bounds. As compared to \eqref{on-average-l2}, the stability bound below involves an additional term $O(\sum_{j=1}^{t}\eta_j^{\frac{2}{1-\alpha}})$, which is the cost we pay by relaxing the smoothness condition to a H\"older continuity of (sub)gradients. It is worth mentioning that our stability bounds apply to non-differentiable loss functions including the popular hinge loss.
\begin{theorem}[Stability bounds]\label{thm:on-average-holder-maintext}
  Assume for all $z\in\zcal$, the map $\bw\mapsto f(\bw;z)$ is nonnegative, convex and $\partial f(\bw;z)$ is $(\alpha,L)$-H\"older continuous with $\alpha\in[0,1)$.
  Let $S,\widetilde{S}$ and $S^{(i)}$ be constructed in Definition \ref{def:aver-stab}. Let $\{\bw_t\}$ and $\{\bw_t^{(i)}\}$ be  produced by \eqref{SGD} based on $S$ and $S^{(i)}$, respectively. Then
  \[
      \ebb_{S,\widetilde{S},A}\Big[\frac{1}{n}\sum_{i=1}^{n}\|\bw_{t+1}-\bw_{t+1}^{(i)}\|_2^2\Big]
      =O\Big(\sum_{j=1}^{t}\eta_j^{\frac{2}{1-\alpha}}\Big)+
      O\Big(n^{-1}(1+t/n)\sum_{j=1}^{t}\eta_j^2\ebb_{S,A}\Big[F_S^{\frac{2\alpha}{1+\alpha}}(\bw_j)\Big]\Big).
  \]
\end{theorem}

\textbf{Generalization bounds}.
We now present generalization bounds for learning by loss functions with H\"older continuous (sub)gradients, which are specific instantiations of a general result (Proposition \ref{prop:gen-holder-up}) stated and proved in Appendix \ref{sec:proof-holder-gen}.
\begin{theorem}[Generalization bounds]\label{thm:error-holder-cor-a}
  Assume for all $z\in\zcal$, the function $\bw\mapsto f(\bw;z)$ is nonnegative, convex, and $\partial f(\bw;z)$ is $(\alpha,L)$-H\"older continuous with $\alpha\in[0,1)$. Let $\{\bw_t\}_t$ be given by \eqref{SGD} with $\eta_t=cT^{-\theta},\theta\in[0,1],c>0$.
  \begin{enumerate}[(a)]
    \item If $\alpha\geq1/2$, we can take $\theta=1/2$ and $T\asymp n$ to derive $\ebb_{S,A}[F(\bw_T^{(1)})]-F(\bw^*)=O(n^{-\frac{1}{2}})$.
    \item If $\alpha<1/2$, we can take $T\asymp n^{\frac{2-\alpha}{1+\alpha}}$ and $\theta=\frac{3-3\alpha}{2(2-\alpha)}$ to derive $\ebb_{S,A}[F(\bw_T^{(1)})]-F(\bw^*)=O(n^{-\frac{1}{2}})$.
    \item If $F(\bw^*)=0$, we take $T\asymp n^{\frac{2}{1+\alpha}}$ and $\theta=\frac{3-\alpha^2-2\alpha}{4}$ to derive
    $\ebb_{S,A}[F(\bw_T^{(1)})]-F(\bw^*)=O(n^{-\frac{1+\alpha}{2}})$.
  \end{enumerate}
\end{theorem}
\begin{remark}
  Although relaxing smoothness affects stability by introducing $O(\sum_{j=1}^{t}\eta_j^{\frac{2}{1-\alpha}})$ in the stability bound, we achieve a generalization bound similar to the smooth case with a similar computation cost if $\alpha\geq1/2$. For $\alpha<1/2$, a minimax optimal generalization bound $O(n^{-\frac{1}{2}})$~\citep{agarwal2012information} can be also achieved with more computation cost as $T\asymp n^{\frac{2-\alpha}{1+\alpha}}$.
  {\red In particular, if $\alpha=0$ we develop the optimal generalization bounds $O(n^{-\frac{1}{2}})$ for SGD with $T\asymp n^2$ iterations. To our best knowledge, this gives the first generalization bounds for SGD with non-smooth loss functions (e.g., hinge loss) based on stability analysis.}
  Analogous to the smooth case, we can derive generalization bounds better than $O(n^{-\frac{1}{2}})$ in the case with low noises. To our best knowledge, this is the first optimistic generalization bound for SGD with non-smooth loss functions. %, which has not been studied yet even in the ERM setting ignoring optimization errors.
\end{remark}

\begin{remark}
  We can extend our discussion to ERM. If $F_S$ is $\sigma$-strongly convex and $\partial f(\bw;z)$ is $(\alpha,L)$-H\"older continuous, we can apply the on-average model stability to show  (see Proposition \ref{prop:erm})\[\ebb_S\big[F(A(S))-F_S(A(S))\big]=O(\ebb_{S}\big[F^{\frac{2\alpha}{1+\alpha}}(A(S))\big]/(n\sigma)),\]where $A(S)=\arg\min_{\bw\in\rbb^d}F_S(\bw)$.
  This extends the error bounds developed for ERM with strongly-smooth loss functions~\citep{srebro2010smoothness,shalev2014understanding} to the non-smooth case, and removes the $G$-admissibility assumption in \citet{bousquet2002stability}. In a low-noise case with a small $\ebb_{S}\big[F(A(S))\big]$, the discussion based on an on-average stability can imply optimistic generalization bounds for ERM.
  %As a comparison, ERM was shown to be $O(G^2/(n\sigma))$-uniformly stable under an additional assumption on the $G$-admissibility of loss functions~\citep{bousquet2002stability}, which is closely related to Assumption \ref{ass:lipschitz}.
  %In a low noise case with small $\ebb_{S}\big[F(A(S))\big]$, our approach can imply optimistic generalization bounds for ERM, while the uniform-stability bounds involving the admissibility constant $G$ fail to exploit the low noise condition.
  %This also shows the effectiveness of our approach in deriving optimistic bounds for ERM if $\ebb_{S}\big[F(A(S))\big]$ is small.
\end{remark}
\section{Stability with Relaxed Convexity\label{sec:convex}}
%\subsection{Stability and generalization errors}
We now turn to stability and generalization of SGD for learning problems where the empirical objective $F_S$ is convex but each loss function $f(\bw;z)$ may be non-convex. For simplicity, we impose Assumption \ref{ass:lipschitz} here and use the arguments based on the uniform stability.
The proofs of Theorem \ref{thm:stab-convex} and Theorem \ref{thm:error-convex} are given in Appendix \ref{sec:proof-convex}.
\begin{theorem}\label{thm:stab-convex}
  Let Assumption \ref{ass:lipschitz} hold.
  Assume for all $z\in\zcal$, the function $\bw\mapsto f(\bw;z)$ is $L$-smooth.
  Let $S=\{z_1,\ldots,z_n\}$ and $\widetilde{S}=\{\tilde{z}_1,\ldots,\tilde{z}_n\}$ be two sets of training examples that differ by a single example.
  Let $\{\bw_t\}_t$ and $\{\tilde{\bw}_t\}_t$ be produced by \eqref{SGD} based on $S$ and $\widetilde{S}$, respectively.  If for all $S$, $F_S$ is convex, then
  \[
  \Big(\ebb_A\big[\|\bw_{t+1}\!-\!\tilde{\bw}_{t+1}\|_2^2\big]\Big)^{\frac{1}{2}}\!\leq \! 4GC_t\sum_{j=1}^{t}\frac{\eta_j}{n}\!+\!2G\Big(C_t\sum_{j=1}^{t}\frac{\eta_j^2}{n}\Big)^{\frac{1}{2}},
  \]
  where we introduce $C_t=\prod_{\tilde{j}=1}^{t}\Big(1+L^2\eta_{\tilde{j}}^2\Big)$.
\end{theorem}
\begin{remark}
  The derivation of uniform stability bounds in \citet{hardt2016train} is based on the non-expansiveness of the operator $\bw\mapsto\bw-\partial f(\bw;z)$, which requires the convexity of $\bw\mapsto f(\bw;z)$ for all $z$.
  Theorem \ref{thm:stab-convex} relaxes this convexity condition to a milder convexity condition on $F_S$.
  If $\sum_{j=1}^{\infty}\eta_j^2<\infty$, the stability bounds in Theorem~\ref{thm:stab-convex} become $O(n^{-1}\sum_{j=1}^{t}\eta_j+n^{-\frac{1}{2}})$ since $C_t<\infty$. % in this case. %It should be mentioned that our result is slightly stronger since $\ebb_A\big[\|\bw_{t+1}-\tilde{\bw}_{t+1}\|_2\big]\leq \big(\ebb_A\big[\|\bw_{t+1}-\tilde{\bw}_{t+1}\|_2^2\big]\big)^{\frac{1}{2}}$.
\end{remark}
%\begin{remark}\label{rem:non-convex}
  By the proof, Theorem \ref{thm:stab-convex} holds if the convexity condition is replaced by
  %\begin{multline*}
  $\ebb_A\big[\langle\bw_t-\tilde{\bw}_t,\partial F_S(\bw_t)-\partial F_S(\tilde{\bw}_t)\rangle\big]
  \geq -C\eta_t\ebb_A[\|\bw_t-\tilde{\bw}_t\|_2^2]$
  %\end{multline*}
  for some $C$ and all $t\in\nbb$.
%\end{remark}

%Now we consider generalization bounds for learning with convex empirical objectives but possibly non-convex individual components.
%The following bounds are immediate by combining the stability bounds and optimization error bounds together.
As shown below, minimax optimal generalization bounds can be achieved for step sizes $\eta_t=\eta_1t^{-\theta}$ for all $\theta\in(1/2,1)$
as well as the step sizes $\eta_t\asymp1/\sqrt{T}$ with $T\asymp n$.
\begin{theorem}\label{thm:error-convex}
  Let Assumption \ref{ass:lipschitz} hold. Assume for all $z\in\zcal$, the function $\bw\mapsto f(\bw;z)$ is $L$-smooth.
  Let $\{\bw_t\}_t$ be produced by \eqref{SGD}. Suppose for all $S$, $F_S$ is convex.
  \begin{enumerate}[(a)]
    \item If $\eta_t=\eta_1t^{-\theta},\theta\in(1/2,1)$, then
    \[
    \!\!\!\ebb_{S,A}[F(\bw_T^{(1)})]-F(\bw^*)=O\Big(n^{-1}T^{1-\theta}+n^{-\frac{1}{2}}+T^{\theta-1}\Big).
    \]
  If $T\!\asymp\! n^{\frac{1}{2-2\theta}}$, then $\ebb_{S,A}[F(\bw_T^{(1)})]\!-\!F(\bw^*)=O(n^{-\frac{1}{2}})$.
  \item If $\eta_t=c/\sqrt{T}$ for some $c>0$ and $T\asymp n$, then $\ebb_{S,A}[F(\bw_T^{(1)})]-F(\bw^*)=O(n^{-\frac{1}{2}})$.
  \end{enumerate}
\end{theorem}
%According to Theorem \ref{thm:error-convex}, we see that optimal generalization bounds can be achieved for step sizes $\eta_t=\eta_1t^{-\theta}$ for all $\theta>1/2$
%as well as the step sizes $\eta_t\asymp1/\sqrt{T}$.
%However, the associated computation complexity required is different. Roughly speaking, we require an increasing computation cost as $\theta$ approaches to $1$. If we %take the step sizes $\eta_t=\eta_1\big(t\log^\beta(et)\big)^{-\frac{1}{2}}$, then we can achieve almost optimal generalization bounds (up to a constant factor) with the %computation cost $T\asymp n$.

%
%\subsection{Application: Stochastic AUC Maximization\label{sec:auc-max}}

\textbf{Example: AUC Maximization}. We now consider a specific example of AUC (Area under ROC curve) maximization where the objective function is convex but each loss function may be non-convex. As a widely used method in imbalanced classification ($\ycal=\{+1,-1\}$), AUC maximization was often formulated as a pairwise learning problem where the corresponding loss function involves a pair of training examples~\citep{zhao2011online,gao2013one}. Recently, AUC maximization algorithms updating models with a single example per iteration were developed~\citep{ying2016stochastic,natole2018stochastic,liu2018fast}. Specifically, AUC maximization with the square loss can be formulated as the minimization of the following objective function
\begin{equation}\label{auc-objective}
\min_{\bw\in\Omega}F(\bw):=p(1-p)\ebb\big[\big(1-\bw^\top(x-\tilde{x})\big)^2|y=1,\tilde{y}=-1\big],
\end{equation}
where $p=\mathrm{Pr}\{Y=1\}$ is the probability of a example being positive.
Let $x_+=\ebb[X|Y=1]$ and $x_-=\ebb[X|Y=-1]$ be the conditional expectation of $X$ given $Y=1$ and $Y=-1$, respectively.
It was shown that $\ebb_{i_t}\big[f(\bw;z_{i_t})\big]=F(\bw)$ for all $\bw\in\rbb^d$~\citep[][Theorem 1]{natole2018stochastic}, where
\begin{multline}
f(\bw;z)\!=\!(1\!-\!p)\big(\bw^\top(x\!-\!x_+)\big)^2\ibb_{[y=1]}\!+\!p(1\!-\!p)\!%\label{auc-component}
\!+\!2\big(1\!+\!\bw^\top(x_-\!-\!x_+)\big)\bw^\top x\big(p\ibb_{[y=-1]}\!-\!(1\!-\!p)\ibb_{[y=1]}\big)\\
+\!p\big(\bw^\top(x\!-\!x_-)\big)^2\ibb_{[y=-1]}-\!p(1-p)\big(\bw^\top(x_-\!-\!x_+)\big)^2.\label{auc-component}
\end{multline}
An interesting property is that \eqref{auc-component} involves only a single example $z$.
This observation allows \citet{natole2018stochastic} to develop a stochastic algorithm as \eqref{SGD} to solve \eqref{auc-objective}.
However, for each $z$, the function $z\mapsto f(\bw;z)$ is non-convex since the associated Hessian matrix may not be positively definite.
It is clear that its expectation $F$ is convex.
%(details can be found in Appendix \ref{sec:proof-auc}). Therefore, the stability analysis in \citet{hardt2016train} does not apply here.
%Corollary \ref{cor:auc} is a direct application of Theorem~\ref{thm:error-convex}.

\section{Stability with Relaxed Strong Convexity\label{sec:strong}}
\subsection{Stability and generalization errors}
Finally, we consider learning problems with strongly convex empirical objectives but possibly non-convex loss functions.
Theorem \ref{thm:stable-strong}  provides stability bounds, while the minimax optimal generalization bounds $O(1/(\sigma n))$ are presented in Theorem \ref{thm:error-strong}.  The proofs are given in Appendix \ref{sec:proof-strong}.
\begin{theorem}\label{thm:stable-strong}
  Let Assumptions in Theorem \ref{thm:stab-convex} hold. Suppose for all $S\subset\zcal$, $F_S$ is $\sigma_S$-strongly convex.
  Then, there exists a constant $t_0$ such that for SGD with $\eta_t=2/((t+t_0)\sigma_S)$ we have
  \[
  \Big(\ebb_A\big[\|\bw_{t+1}-\tilde{\bw}_{t+1}\|_2^2\big]\Big)^{\frac{1}{2}}\leq \frac{4G}{\sigma_S}\Big(\frac{1}{\sqrt{n(t+t_0)}}+\frac{1}{n}\Big).
  \]
\end{theorem}
\begin{remark}
  Under the assumption $\bw\mapsto f(\bw,z)$ is $\sigma$-strongly convex and smooth for all $z$, it was shown that $\ebb_A\big[\|\bw_{t+1}-\tilde{\bw}_{t+1}\|_2\big]=O(1/(n\sigma))$ for $\eta_t=O(1/(\sigma t))$~\citep{hardt2016train}. Indeed, this strong convexity condition is used to show that the operator $\bw\mapsto\bw-\partial f(\bw;z)$ is contractive.
  We relax the strong convexity of $f(\bw;z)$ to the strong convexity of $F_S$.
  Our stability bound holds even if $\bw\mapsto f(\bw;z)$ is non-convex. If $t\asymp n$, then our stability bound coincides with the one in \citet{hardt2016train} up to a constant factor.
\end{remark}

\begin{theorem}\label{thm:error-strong}
  Let Assumption \ref{ass:lipschitz} hold. Assume for all $z\in\zcal$, the function $\bw\mapsto f(\bw;z)$ is $L$-smooth. Suppose for all $S\subset\zcal$, $F_S$ is $\sigma_S$-strongly convex.
  Then, there exists some $t_0$ such that for SGD with $\eta_t=2/((t+t_0)\sigma_S)$ and $T\asymp n$ we have
  \[
  \ebb_{S,A}[F(\bw_T^{(2)})]-F(\bw^*)=O(\ebb_S\big[1/(n\sigma_S)\big]),
  \]
  where $\bw_T^{(2)}=\big(\sum_{t=1}^{T}(t+t_0-1)\bw_t\big)/\sum_{t=1}^{T}(t+t_0-1)$.
\end{theorem}

\subsection{Application: least squares regression}
We now consider an application to learning with the least squares loss, where $f(\bw;z)=\frac{1}{2}\big(\langle\bw,x\rangle-y\big)^2$. Let $\Omega=\{\bw\in\rbb^d:\|\bw\|_2\leq R\}$. In this case, \eqref{SGD} becomes
\begin{equation}\label{SGD-LS}
  \bw_{t+1}=\Pi_\Omega\big(\bw_t-\eta_t\big(\langle\bw_t,x_t\rangle-y_t\big)x_t\big),
\end{equation}
where $\Pi_\Omega(\bw)=\min\{R/\|\bw\|_2,1\}\bw$.
Note that each individual loss function $f(\bw_t;z_t)$ is non-strongly convex. However, as we will show below, the empirical objective satisfies a strong convexity on a subspace containing the iterates $\{\bw_t\}$.
For any $S=\{z_1,\ldots,z_n\}$ let $C_S=\frac{1}{n}\sum_{i=1}^{n}x_ix_i^\top$ be the empirical covariance matrix and $\sigma_S'$ be the minimal positive eigenvalue of $C_S$. Then it is clear from \eqref{SGD-LS} that $\{\bw_t\}_t$ belongs to the range of $C_S$.~\footnote{The range of $C_S$ is the linear span of $x_1,\ldots,x_n$. Details are given in Proposition \ref{prop:span} in Appendix \ref{sec:proof-strong}.} Let $\widetilde{S}\subset\zcal^n$ differ from $S$ by a single example. For simplicity, we assume $S$ and $\widetilde{S}$ differ by the first example and denote $\widetilde{S}=\{\tilde{z}_1,z_2,\ldots,z_n\}$. We construct a set $\bar{S}=\{0,z_2,\ldots,z_n\}$. Let $\{\bw_t\},\{\tilde{\bw}_t\}$ and $\{\bar{\bw}_t\}$ be the sequence by \eqref{SGD-LS} based on $S,\widetilde{S}$ and $\bar{S}$, respectively. Then our previous discussion implies that $\bw_t-\bar{\bw}_t\in\ran(C_S),\tilde{\bw}_t-\bar{\bw}_t\in\ran(C_{\widetilde{S}})$ for all $t\in\nbb$ ($\ran(C_{\bar{S}})\subseteq\ran(C_S), \ran(C_{\bar{S}})\subseteq\ran(C_{\widetilde{S}})$), where we denote by $\ran(C)$ the range of a matrix $C$. It follows that $\bw_t-\bar{\bw}_t$ and $\tilde{\bw}_t-\bar{\bw}_t$ are orthogonal to the kernel of $C_S$ and $C_{\widetilde{S}}$, respectively. Therefore,
\[
\langle \bw_t-\bar{\bw}_t, C_S(\bw_t-\bar{\bw}_t)\rangle\geq \sigma_S'\|\bw_t-\bar{\bw}_t\|_2^2,
\]
\[
\langle \tilde{\bw}_t-\bar{\bw}_t, C_{\widetilde{S}}(\tilde{\bw}_t-\bar{\bw}_t)\rangle\geq \sigma_{\widetilde{S}}'\|\tilde{\bw}_t-\bar{\bw}_t\|_2^2.
\]
As we will see in the proof, Theorem \ref{thm:stable-strong} holds if only the following local strong convexity holds, i.e.,
\[
\langle\bw_t-\tilde{\bw}_t,\partial F_S(\bw_t)-\partial F_S(\tilde{\bw}_t)\rangle\geq\sigma_S\|\bw_t-\tilde{\bw}_t\|_2^2,\;\forall t\in\nbb.
\]
Therefore, we can apply Theorem \ref{thm:stable-strong} with $\widetilde{S}=\bar{S}$ and $\sigma_S=\sigma_S'$ to derive (note $\partial F_S(\bw)=C_S\bw-\frac{1}{n}\sum_{i=1}^{n}y_ix_i$)
\[
\ebb_A\big[\|\bw_{t+1}-\bar{\bw}_{t+1}\|_2\big]\leq \frac{4G}{\sigma_S'}\Big(\frac{1}{\sqrt{n(t+t_0)}}+\frac{1}{n}\Big).
\]
A similar inequality also holds for $\ebb_A\big[\|\tilde{\bw}_{t+1}-\bar{\bw}_{t+1}\|_2\big]$,
which together with the subadditivity of $\|\cdot\|_2$ immediately gives the following stability bound on $\ebb_A[\|\bw_{t+1}-\tilde{\bw}_{t+1}\|_2]$.
\begin{corollary}
Let $f(\bw;z)=\frac{1}{2}\big(\langle\bw,x\rangle-y\big)^2$ and $\Omega=\{\bw\in\rbb^d:\|\bw\|_2\leq R\}$ for some $R>0$. There exists an $t_0$ such that for \eqref{SGD-LS} with $\eta_j=2/(\sigma_S'(j+t_0))$ we have
  \[
  \ebb_A\big[\|\bw_{t}-\tilde{\bw}_{t}\|_2\big]=O\Big(\frac{1}{\sqrt{n(t+t_0)}\sigma_S'}+\frac{1}{n\sigma_S'}\Big).%O\Big(1/(n\sigma_S')+1/(n\sigma_{\widetilde{S}}')\Big)\quad\text{and}\quad
  %\ebb_{S,A}[F_S(\bw_T^{(2)})]-F(\bw^*)=O(\ebb_S\big[1/(n\sigma_S')\big].
  \]
\end{corollary}

%Other than least squares regression, our results can be also applied to a regularized version of stochastic AUC maximization in Section \ref{sec:auc-max}, where the empirical objective is strongly convex but each loss function may be non-convex. We omit this discussion due to the page limit.

\section{Conclusions\label{sec:conclusion}}
In this paper, we study stability and generalization of SGD by removing the bounded gradient assumptions, and relaxing the smoothness assumption and the convexity requirement of each loss function in the existing analysis.
We introduce a novel on-average model stability able to capture the risks of SGD iterates, which implies fast generalization bounds in the low-noise case and {\red stability bounds for learning with even \emph{non-smooth} loss functions.}
For all considered problems, we show that our stability bounds can imply minimax optimal generalization bounds by balancing optimization and estimation errors.
We apply our results to practical learning problems to justify the superiority of our approach over the existing stability analysis. Our results can be extended to stochastic proximal gradient descent, high-probability bounds and SGD without replacement (details are given in Appendix \ref{sec:extension}).
In the future, it would be interesting to study stability bounds for other stochastic optimization algorithms, e.g., Nesterov's accelerated variants of SGD~\citep{nesterov2013introductory}.
%It is also very interesting to improve the existing stability bounds in the general non-convex case~\citep{hardt2016train}.

\section*{Acknowledgement}
The work of Y. Lei is supported by the National Natural Science Foundation of China (Grant Nos. 61806091, 11771012) and the Alexander von Humboldt Foundation.
The work of Y. Ying is supported by the National Science Foundation (NSF) under Grant No. \#1816227.

\appendix
\numberwithin{equation}{section}
\numberwithin{theorem}{section}
\numberwithin{figure}{section}
\numberwithin{table}{section}
\renewcommand{\thesection}{{\Alph{section}}}
\renewcommand{\thesubsection}{\Alph{section}.\arabic{subsection}}
\renewcommand{\thesubsubsection}{\Roman{section}.\arabic{subsection}.\arabic{subsubsection}}
\setcounter{secnumdepth}{-1}
\setcounter{secnumdepth}{3}
\section{Optimization Error Bounds}
For a full picture of generalization errors, we need to address the optimization errors. This is achieved by the following lemma.
Parts (a) and (b) consider the convex and strongly convex empirical objectives, respectively (we make no assumptions on the convexity of each $f(\cdot,z)$).
Note in Parts (c) and (d), we do not make a bounded gradient assumption. As an alternative, we require convexity of $f(\cdot;z)$ for all $z$.
An appealing property of Parts (c) and (d) is that it involves $O(\sum_{j=1}^{t}\eta_j^2F_S(\bw))$ instead of $O(\sum_{j=1}^{t}\eta_j^2)$, which is a requirement for developing fast rates in the case with low noises.
%It should be noted that we only need (strong) convexity on the empirical objective.

Our discussion on optimization errors requires to use a self-bounding property for functions with H\"older continuous (sub)gradients, which means that gradients can be controlled by function values. The case $\alpha=1$ was established in \citet{srebro2010smoothness}.
The case $\alpha\in(0,1)$ was established in \citet{ying2017unregularized}.
The case $\alpha=0$ follows directly from Definition \ref{def:holder}.
Define
\begin{equation}\label{alpha-1}
  c_{\alpha,1}=\begin{cases}
                 (1+1/\alpha)^{\frac{\alpha}{1+\alpha}}L^{\frac{1}{1+\alpha}}, & \mbox{if } \alpha>0 \\
                 \sup_z\|\partial f(0;z)\|_2+L, & \mbox{if } \alpha=0.
               \end{cases}
\end{equation}
%\begin{lemma}[\label{lem:self-bounding}\citealt{srebro2010smoothness}]
%  If $g$ is nonnegative and $L$-smooth, then
%  \[
%  \|\nabla g(\bw)\|_2^2\leq 2Lg(\bw),\quad\forall \bw\in\rbb^d.
%  \]
%\end{lemma}
\begin{lemma}\label{lem:self-bounding}
Assume for all $z\in\zcal$, the map $\bw\mapsto f(\bw;z)$ is nonnegative,  and $\bw\mapsto \partial f(\bw;z)$ is $(\alpha,L)$-H\"older continuous with $\alpha\in[0,1]$. Then for $c_{\alpha,1}$ defined as \eqref{alpha-1} we have
  \[
  \|\partial f(\bw,z)\|_2\leq c_{\alpha,1}f^{\frac{\alpha}{1+\alpha}}(\bw,z),\quad\forall \bw\in\rbb^d,z\in\zcal.
  \]
\end{lemma}
\begin{lemma}\label{lem:computation-ave}
  \begin{enumerate}[(a)]
    \item Let $\{\bw_t\}_t$ be produced by \eqref{SGD} and Assumption \ref{ass:lipschitz} hold. If $F_S$ is convex, then for all $t\in\nbb$ and $\bw\in\Omega$
    \[
    \ebb_A[F_S(\bw_t^{(1)})]-F_S(\bw)\leq \frac{G^2\sum_{j=1}^{t}\eta_j^2+\|\bw\|_2^2}{2\sum_{j=1}^{t}\eta_j},
    \]
    where $\bw_t^{(1)}=\big(\sum_{j=1}^{t}\eta_j\bw_j\big)/\sum_{j=1}^{t}\eta_j$.
    \item Let $F_S$ be $\sigma_S$-strongly convex and Assumption \ref{ass:lipschitz} hold. Let $t_0\geq0$ and $\{\bw_t\}_t$ be produced by \eqref{SGD} with $\eta_t=2/(\sigma_S(t+t_0))$. Then for all $t\in\nbb$ and $\bw\in\Omega$
    \[
    \ebb_A[F_S(\bw_t^{(2)})]-F_S(\bw)=O\big(1/(t\sigma_S)+\|\bw\|_2^2/t^2\big),
    \]
    where $\bw_t^{(2)}=\big(\sum_{j=1}^{t}(j+t_0-1)\bw_j\big)/\sum_{j=1}^{t}(j+t_0-1)$.
    \item Assume for all $z\in\zcal$, the function $\bw\mapsto f(\bw;z)$ is nonnegative, convex and $L$-smooth. Let $\{\bw_t\}_t$ be produced by \eqref{SGD} with $\eta_t\leq 1/(2L)$. If the step size is nonincreasing, then for all $t\in\nbb$ and $\bw\in\Omega$ independent of the SGD algorithm $A$
        \[
        \sum_{j=1}^{t}\eta_j\ebb_A[F_S(\bw_j)-F_S(\bw)]  \leq (1/2+L\eta_1)\|\bw\|_2^2 + 2L\sum_{j=1}^{t}\eta_j^2F_S(\bw).
        \]
%        \[
%    \ebb_A[F_S(\bw_t^{(1)})]-F_S(\bw)\leq \frac{(1/2+L\eta_1)\|\bw\|_2^2 + 2L\sum_{j=1}^{t}\eta_j^2\ebb_A[F_S(\bw)]}{\sum_{j=1}^{t}\eta_j}.
%    \]
    \item Assume for all $z\in\zcal$, the function $\bw\mapsto f(\bw;z)$ is nonnegative, convex, and $\partial f(\bw;z)$ is $(\alpha,L)$-H\"older continuous with $\alpha\in[0,1)$. Let $\{\bw_t\}_t$ be produced by \eqref{SGD} with nonincreasing step sizes. Then for all $t\in\nbb$ and $\bw\in\Omega$ independent of the SGD algorithm $A$ we have
        \[
        2\sum_{j=1}^{t}\eta_j\ebb_A[F_S(\bw_j)-F_S(\bw)]\leq \|\bw\|_2^2+
        c_{\alpha,1}^2\Big(\sum_{j=1}^{t}\eta_j^2\Big)^{\frac{1-\alpha}{1+\alpha}}
   \Big(\eta_1\|\bw\|_2^2+2\sum_{j=1}^{t}\eta_j^2F_S(\bw)+c_{\alpha,2}\sum_{j=1}^{t}\eta_j^{\frac{3-\alpha}{1-\alpha}}\Big)^{\frac{2\alpha}{1+\alpha}},
        \]
        where
        \begin{equation}\label{alpha-2}
          c_{\alpha,2}=\begin{cases}
                         \frac{1-\alpha}{1+\alpha}(2\alpha/(1+\alpha))^{\frac{2\alpha}{1-\alpha}}c_{\alpha,1}^{\frac{2+2\alpha}{1-\alpha}}, & \mbox{if } \alpha>0 \\
                         c_{\alpha,1}^2, & \mbox{if }\alpha=0.
                       \end{cases}
        \end{equation}

  \end{enumerate}
\end{lemma}
\begin{proof}
   Parts (a) and (b) can be found in the literature~\citep{nemirovski2009robust,lacoste2012simpler}. We only prove Parts (c) and (d).
   We first prove Part (c).
   The projection operator $\Pi_\Omega$ is non-expansive, i.e.,
  \begin{equation}\label{proj-cont}
    \big\|\Pi_\Omega(\bw)-\Pi_\Omega(\tilde{\bw})\big\|_2 \leq \|\bw-\tilde{\bw}\|_2.
  \end{equation}
   By the SGD update \eqref{SGD}, \eqref{proj-cont}, convexity and Lemma \ref{lem:self-bounding}, we know
   \begin{align}
     \|\bw_{t+1}-\bw\|_2^2 & \leq \|\bw_t-\eta_t\partial f(\bw_t;z_{i_t})-\bw\|_2^2\notag\\
      & = \|\bw_t-\bw\|_2^2 + \eta_t^2\|\partial f(\bw_t;z_{i_t})\|_2^2 + 2\eta_t\langle\bw-\bw_t,\partial f(\bw_t;z_{i_t}) \notag\\
      & \leq \|\bw_t-\bw\|_2^2 + 2\eta_t^2Lf(\bw_t;z_{i_t}) + 2\eta_t(f(\bw;z_{i_t})-f(\bw_t;z_{i_t}))\label{computation-ave-1}\\
      & \leq \|\bw_t-\bw\|_2^2+2\eta_tf(\bw;z_{i_t})-\eta_tf(\bw_t;z_{i_t}),\notag
   \end{align}
   where the last inequality is due to $\eta_t\leq1/(2L)$. It then follows that
   \[
   \eta_tf(\bw_t;z_{i_t})\leq \|\bw_t-\bw\|_2^2-\|\bw_{t+1}-\bw\|_2^2+2\eta_tf(\bw;z_{i_t}).
   \]
   Multiplying both sides by $\eta_t$ and using the assumption $\eta_{t+1}\leq \eta_t$, we know
   \begin{align*}
     \eta_t^2f(\bw_t;z_{i_t}) & \leq \eta_t\|\bw_t-\bw\|_2^2-\eta_t\|\bw_{t+1}-\bw\|_2^2+2\eta_t^2f(\bw;z_{i_t})\\
     & \leq \eta_t\|\bw_t-\bw\|_2^2-\eta_{t+1}\|\bw_{t+1}-\bw\|_2^2+2\eta_t^2f(\bw;z_{i_t}).
   \end{align*}
   Taking a summation of the above inequality gives ($\bw_1=0$)
   \[
   \sum_{j=1}^{t}\eta_j^2f(\bw_j;z_{i_j})\leq\eta_1\|\bw\|_2^2+2\sum_{j=1}^{t}\eta_j^2f(\bw;z_{i_j}).
   \]
   Taking an expectation w.r.t. $A$ gives (note $\bw_j$ is independent of $i_j$)
   \begin{equation}\label{computation-ave-2}
   \sum_{j=1}^{t}\eta_j^2\ebb_A[F_S(\bw_j)]=\sum_{j=1}^{t}\eta_j^2\ebb_A\big[f(\bw_j;z_{i_j})\big]\leq\eta_1\|\bw\|_2^2+2\sum_{j=1}^{t}\eta_j^2\ebb_A[F_S(\bw)].
   \end{equation}
   On the other hand, taking an expectation w.r.t. $i_t$ over both sides of \eqref{computation-ave-1} shows
   \[
   2\eta_t\big[F_S(\bw_t)-F_S(\bw)\big]\leq \|\bw_t-\bw\|_2^2-\ebb_{i_t}\big[\|\bw_{t+1}-\bw\|_2^2\big]+2\eta_t^2LF_S(\bw_t).
   \]
   Taking an expectation on both sides followed with a summation, we get
   \begin{align*}
   2\sum_{j=1}^{t}\eta_j\ebb_A[F_S(\bw_j)-F_S(\bw)] & \leq \|\bw\|_2^2 + 2L\sum_{j=1}^{t}\eta_j^2\ebb_A[F_S(\bw_j)] \\
   & \leq (1+2L\eta_1)\|\bw\|_2^2 + 4L\sum_{j=1}^{t}\eta_j^2\ebb_A[F_S(\bw)],
   \end{align*}
   where the last step is due to \eqref{computation-ave-2}.
   The proof is complete since $\bw$ is independent of $A$.

   We now prove Part (d). Analogous to \eqref{computation-ave-1}, one can show for loss functions with H\"older continuous (sub)gradients (Lemma \ref{lem:self-bounding})
   \begin{equation}\label{computation-ave-3}
   \|\bw_{t+1}-\bw\|_2^2 \leq \|\bw_t-\bw\|_2^2 + c_{\alpha,1}^2\eta_t^2f^{\frac{2\alpha}{1+\alpha}}(\bw_t;z_{i_t})+2\eta_t(f(\bw;z_{i_t})-f(\bw_t;z_{i_t})).
   \end{equation}
   By the Young's inequality
  \begin{equation}\label{young}
  ab\leq p^{-1}|a|^p+q^{-1}|b|^q,\quad a,b\in\rbb,p,q>0 \text{ with } p^{-1}+q^{-1}=1,
  \end{equation} we know (notice the following inequality holds trivially if $\alpha=0$)
   \begin{align*}
     \eta_tc_{\alpha,1}^2f^{\frac{2\alpha}{1+\alpha}}(\bw_t;z_{i_t}) & =\Big(\frac{1+\alpha}{2\alpha}f(\bw_t;z_{i_t})\Big)^{\frac{2\alpha}{1+\alpha}}\Big(\frac{2\alpha}{1+\alpha}\Big)^{\frac{2\alpha}{1+\alpha}} c_{\alpha,1}^2\eta_t\\
      & \leq \frac{2\alpha}{1+\alpha}\Big(\frac{1+\alpha}{2\alpha}f(\bw_t;z_{i_t})\Big)^{\frac{2\alpha}{1+\alpha}\frac{1+\alpha}{2\alpha}}+
      \frac{1-\alpha}{1+\alpha}\Big(\Big(\frac{2\alpha}{1+\alpha}\Big)^{\frac{2\alpha}{1+\alpha}} c^2_{\alpha,1}\eta_t\Big)^{\frac{1+\alpha}{1-\alpha}}\\
      & = f(\bw_t;z_{i_t})+c_{\alpha,2}\eta_t^{\frac{1+\alpha}{1-\alpha}}.
   \end{align*}
   Combining the above two inequalities together, we get
   \[
   \eta_tf(\bw_t;z_{i_t})\leq \|\bw_t-\bw\|_2^2 - \|\bw_{t+1}-\bw\|_2^2 + 2\eta_tf(\bw;z_{i_t})+c_{\alpha,2}\eta_t^{\frac{2}{1-\alpha}}.
   \]
   Multiplying both sides by $\eta_t$ and using $\eta_{t+1}\leq\eta_t$, we derive
   \[
   \eta_t^2f(\bw_t;z_{i_t})\leq \eta_t\|\bw_t-\bw\|_2^2 - \eta_{t+1}\|\bw_{t+1}-\bw\|_2^2 + 2\eta_t^2f(\bw;z_{i_t})+c_{\alpha,2}\eta_t^{\frac{3-\alpha}{1-\alpha}}.
   \]
   Taking a summation of the above inequality gives
   \begin{equation}\label{computation-ave-4}
     \sum_{j=1}^{t}\eta_j^2f(\bw_j;z_{i_j})\leq\eta_1\|\bw\|_2^2+2\sum_{j=1}^{t}\eta_j^2f(\bw;z_{i_j})+c_{\alpha,2}\sum_{j=1}^{t}\eta_j^{\frac{3-\alpha}{1-\alpha}}.
   \end{equation}
   According to the Jensen's inequality and the concavity of $x\mapsto x^{\frac{2\alpha}{1+\alpha}}$, we know
   \begin{align}
     \sum_{j=1}^{t}\eta_j^2f^{\frac{2\alpha}{1+\alpha}}(\bw_j;z_{i_j}) & \leq \sum_{j=1}^{t}\eta_j^2\Big(\frac{\sum_{j=1}^{t}\eta_j^2f(\bw_j;z_{i_j})}{\sum_{j=1}^{t}\eta_j^2}\Big)^{\frac{2\alpha}{1+\alpha}} \notag\\
      & \leq \Big(\sum_{j=1}^{t}\eta_j^2\Big)^{\frac{1-\alpha}{1+\alpha}}\Big(\eta_1\|\bw\|_2^2+2\sum_{j=1}^{t}\eta_j^2f(\bw;z_{i_j})+c_{\alpha,2}\sum_{j=1}^{t}\eta_j^{\frac{3-\alpha}{1-\alpha}}\Big)^{\frac{2\alpha}{1+\alpha}},
      \label{computation-ave-5}
   \end{align}
   where in the last step we have used \eqref{computation-ave-4}. Taking an expectation on both sides of \eqref{computation-ave-3}, we know
   \[
   2\eta_t\ebb_A\big[F_S(\bw_t)-F_S(\bw)\big]\leq\ebb_A[\|\bw_t-\bw\|_2^2]-\ebb_A\big[\|\bw_{t+1}-\bw\|_2^2\big]+c_{\alpha,1}^2\eta_t^2\ebb_A\big[f^{\frac{2\alpha}{1+\alpha}}(\bw_t;z_{i_t})\big].
   \]
   Taking a summation of the above inequality gives
   \begin{align*}
   &2\sum_{j=1}^{t}\eta_j\ebb_A[F_S(\bw_j)-F_S(\bw)]\leq \|\bw\|_2^2+c_{\alpha,1}^2\sum_{j=1}^{t}\eta_j^2\ebb_A\big[f^{\frac{2\alpha}{1+\alpha}}(\bw_j;z_{i_j})\big]\\
   &\leq \|\bw\|_2^2+c_{\alpha,1}^2\Big(\sum_{j=1}^{t}\eta_j^2\Big)^{\frac{1-\alpha}{1+\alpha}}
   \Big(\eta_1\|\bw\|_2^2+2\sum_{j=1}^{t}\eta_j^2\ebb_A[F_S(\bw)]+c_{\alpha,2}\sum_{j=1}^{t}\eta_j^{\frac{3-\alpha}{1-\alpha}}\Big)^{\frac{2\alpha}{1+\alpha}},
   \end{align*}
   where we have used \eqref{computation-ave-5} and the concavity of $x\mapsto x^{\frac{2\alpha}{1+\alpha}}$ in the last step.
   The proof is complete by noting the independence between $\bw$ and $A$.
\end{proof}

\section{Proofs on Generalization by On-average Model Stability\label{sec:proof-gen-stab}}
To prove Theorem \ref{thm:gen-model-stab}, we introduce an useful inequality for $L$-smooth functions $\bw\mapsto f(\bw;z)$~\citep{nesterov2013introductory}
\begin{equation}\label{smooth-bound}
  f(\bw;z)\leq f(\tilde{\bw};z)+\langle\bw-\tilde{\bw},\partial f(\tilde{\bw};z)\rangle+\frac{L\|\bw-\tilde{\bw}\|_2^2}{2}.
\end{equation}

%Since $\partial f(\bw;z)$ is $(\alpha,L)$-H\"older continuous, we can apply \eqref{smooth-bound} and \eqref{gen-model-stab-1} to derive

\begin{proof}[Proof of Theorem \ref{thm:gen-model-stab}]
  Due to the symmetry, we know
  \begin{align}
    \ebb_{S,A}\big[F(A(S))-F_S(A(S))\big] & = \ebb_{S,\widetilde{S},A}\Big[\frac{1}{n}\sum_{i=1}^{n}\big(F(A(S^{(i)}))-F_S(A(S))\big)\Big] \notag\\
     & = \ebb_{S,\widetilde{S},A}\Big[\frac{1}{n}\sum_{i=1}^{n}\big(f(A(S^{(i)});z_i)-f(A(S);z_i)\big)\Big], \label{gen-model-stab-1}
  \end{align}
  where the last identity holds since $A(S^{(i)})$ is independent of $z_i$.
  Under Assumption \ref{ass:lipschitz}, it is then clear that
  \[
  \big|\ebb_{S,A}\big[F(A(S))-F_S(A(S))\big]\big|\leq\ebb_{S,\widetilde{S},A}\Big[\frac{G}{n}\sum_{i=1}^{n}\|A(S)-A(S^{(i)})\|_2\Big].
  \]
  This proves Part (a).

  We now prove Part (b). According to \eqref{smooth-bound} due to the $L$-smoothness of $f$ and \eqref{gen-model-stab-1}, we know
  \[
  \ebb_{S,A}\big[F(A(S))-F_S(A(S))\big]
  \leq \frac{1}{n}\sum_{i=1}^{n}\ebb_{S,\widetilde{S},A}\Big[\langle A(S^{(i)})-A(S),\partial f(A(S);z_i)\rangle+\frac{L}{2}\|A(S^{(i)})-A(S)\|_2^2\Big].
  \]
  According to the Schwartz's inequality we know
  \begin{align*}
    \langle A(S^{(i)})-A(S),\partial f(A(S);z_i)\rangle & \leq \|A(S^{(i)})-A(S)\|_2\|\partial f(A(S);z_i)\|_2 \\
     & \leq \frac{\gamma}{2}\|A(S^{(i)})-A(S)\|_2^2 + \frac{1}{2\gamma}\|\partial f(A(S);z_i)\|_2^2\\
     & \leq \frac{\gamma}{2}\|A(S^{(i)})-A(S)\|_2^2 + \frac{L}{\gamma}f(A(S);z_i),
  \end{align*}
  where the last inequality is due to the self-bounding property of smooth functions (Lemma \ref{lem:self-bounding}).
  Combining the above two inequalities together, we derive
  \begin{align*}
    \ebb_{S,A}\big[F(A(S))-F_S(A(S))\big]
    & \leq \frac{L+\gamma}{2n}\sum_{i=1}^{n}\ebb_{S,\widetilde{S},A}\big[\|A(S^{(i)})-A(S)\|_2^2\big] + \frac{L}{n\gamma}\sum_{i=1}^{n}\ebb_{S,A}[f(A(S);z_i)].
  \end{align*}
  The stated inequality in Part (b) then follows directly by noting $\frac{1}{n}\sum_{i=1}^{n}f(A(S);z_i)=F_S(A(S))$.

  Finally, we consider Part (c). By \eqref{gen-model-stab-1} and the convexity of $f$, we know
  \[
  \ebb_{S,A}\big[F(A(S))-F_S(A(S))\big]
  \leq \frac{1}{n}\sum_{i=1}^{n}\ebb_{S,\widetilde{S},A}\Big[\langle A(S^{(i)})-A(S),\partial f(A(S^{(i)});z_i)\rangle\Big].
  \]
  By the Schwartz's inequality and Lemma \ref{lem:self-bounding} we know
  \begin{align*}
    \langle A(S^{(i)})-A(S),\partial f(A(S^{(i)});z_i)\rangle
     & \leq \frac{\gamma}{2}\|A(S^{(i)})-A(S)\|_2^2 + \frac{1}{2\gamma}\|\partial f(A(S^{(i)});z_i)\|_2^2\\
     & \leq \frac{\gamma}{2}\|A(S^{(i)})-A(S)\|_2^2 + \frac{c^2_{\alpha,1}}{2\gamma}f^{\frac{2\alpha}{1+\alpha}}(A(S^{(i)});z_i).
  \end{align*}
  Combining the above two inequalities together, we get
  \[
  \ebb_{S,A}\big[F(A(S))-F_S(A(S))\big]\leq \frac{\gamma}{2n}\sum_{i=1}^{n}\ebb_{S,\widetilde{S},A}\big[\|A(S^{(i)})-A(S)\|_2^2\big]
  +\frac{c^2_{\alpha,1}}{2\gamma}\frac{1}{n}\sum_{i=1}^{n}\ebb_{S,\widetilde{S},A}\big[f^{\frac{2\alpha}{1+\alpha}}(A(S^{(i)});z_i)\big].
  \]
  Since $x\mapsto x^{\frac{2\alpha}{1+\alpha}}$ is concave and $z_i$ is independent of $A(S^{(i)})$, we know
  \begin{align*}
  \ebb_{S,\widetilde{S},A}\big[f^{\frac{2\alpha}{1+\alpha}}(A(S^{(i)});z_i)\big]&\leq \ebb_{S,\widetilde{S},A}\Big[\Big(\ebb_{z_i}\Big[f(A(S^{(i)});z_i)\Big]\Big)^{\frac{2\alpha}{1+\alpha}}\Big]\\
  &=\ebb_{S,\widetilde{S},A}\Big[F^{\frac{2\alpha}{1+\alpha}}(A(S^{(i)}))\Big]=
  \ebb_{S,A}\Big[F^{\frac{2\alpha}{1+\alpha}}(A(S))\Big].
  \end{align*}
  A combination of the above two inequalities then gives the stated bound in Part (c).
  The proof is complete.
\end{proof}

\section{Proof on Learning without Bounded Gradients: Strongly Smooth Case}
\subsection{Stability bounds\label{sec:proof-lip-stab}}
A key property on establishing the stability of SGD is the non-expansiveness of the gradient-update operator established in the following lemma.
\begin{lemma}[\citealt{hardt2016train}\label{lem:nonexpansive}]
Assume for all $z\in\zcal$, the function $\bw\mapsto f(\bw;z)$ is convex and $L$-smooth. Then for $\eta\leq2/L$ we know
  \[
  \|\bw-\eta \partial f(\bw;z)-\tilde{\bw}+\eta \partial f(\tilde{\bw};z)\|_2\leq \|\bw-\tilde{\bw}\|_2.
  \]
\end{lemma}

Based on Lemma \ref{lem:nonexpansive}, we establish stability bounds of models for SGD applied to two sets differing by a single example.
\begin{lemma}\label{lem:on-average}
  Assume for all $z\in\zcal$, the function $\bw\mapsto f(\bw;z)$ is nonnegative, convex and $L$-smooth.
  Let $S,\widetilde{S}$ and $S^{(i)}$ be constructed as Definition \ref{def:aver-stab}. Let $\{\bw_t\}$ and $\{\bw_t^{(i)}\}$ be produced by \eqref{SGD} with $\eta_t\leq2/L$ based on $S$ and $S^{(i)}$, respectively. Then for any $p>0$ we have
  \begin{equation}\label{on-average-a}
  \ebb_{S,\widetilde{S},A}\big[\|\bw_{t+1}-\bw_{t+1}^{(i)}\|_2\big]\leq \frac{2\sqrt{2L}}{n}\sum_{j=1}^{t}\eta_j\ebb_{S,A}\big[\sqrt{f(\bw_j;z_i)}\big]
  \end{equation}
  and
  \begin{equation}\label{on-average-b}
  \ebb_{S,\widetilde{S},A}\big[\|\bw_{t+1}-\bw_{t+1}^{(i)}\|_2^2\big]
  \leq \frac{8(1+1/p)L}{n}\sum_{j=1}^{t}(1+p/n)^{t-j}\eta_j^2\ebb_{S,A}\big[f(\bw_j;z_i)\big].
  \end{equation}
\end{lemma}
\begin{proof}
  If $i_t\neq i$, we know the updates of $\bw_{t+1}$ and $\bw_{t+1}^{(i)}$ are based on stochastic gradients calculated with the same example $z_{i_t}$.
  By Lemma \ref{lem:nonexpansive} we then get
  \begin{equation}\label{on-average-1}
  \|\bw_{t+1}-\bw_{t+1}^{(i)}\|_2\leq \big\|\bw_t-\eta_t\partial f(\bw_t;z_{i_t})-\bw_t^{(i)}+\eta_t\partial f(\bw_t^{(i)},z_{i_t})\big\|_2 \leq \|\bw_{t}-\bw_{t}^{(i)}\|_2.
  \end{equation}
  If $i_t=i$, we know
  \begin{align}
     \|\bw_{t+1}-\bw_{t+1}^{(i)}\|_2 & \leq \big\|\bw_t-\eta_t\partial f(\bw_t;z_i)-\bw_t^{(i)}+\eta_t\partial f(\bw_t^{(i)},\tilde{z}_i)\big\|_2 \notag \\
     & \leq \|\bw_{t}-\bw_{t}^{(i)}\|_2 + \eta_t\|\partial f(\bw_t;z_i)-\partial f(\bw_t^{(i)};\tilde{z}_i)\|_2\label{on-average-2}\\
     & \leq \|\bw_{t}-\bw_{t}^{(i)}\|_2 + \eta_t\|\partial f(\bw_t;z_i)\|_2 + \eta_t\|\partial f(\bw_t^{(i)};\tilde{z}_i)\|_2\notag\\
     & \leq \|\bw_{t}-\bw_{t}^{(i)}\|_2 + \sqrt{2L}\eta_t\Big(\sqrt{f(\bw_t;z_i)}+\sqrt{f(\bw_t^{(i)};\tilde{z}_i)}\Big),\label{on-average-3}
  \end{align}
  where the second inequality follows from the sub-additivity of $\|\cdot\|_2$
  and the last inequality is due to Lemma \ref{lem:self-bounding} on the self-bounding property of smooth functions.

  We first prove Eq. \eqref{on-average-a}.
  Since $i_t$ is drawn from the uniform distribution over $\{1,\ldots,n\}$, we can combine Eqs. \eqref{on-average-1} and \eqref{on-average-3} to derive
  \[
  \ebb_A\big[\|\bw_{t+1}-\bw_{t+1}^{(i)}\|_2\big]\leq \ebb_A\big[\|\bw_t-\bw_t^{(i)}\|\big]+\frac{\sqrt{2L}\eta_t}{n}\ebb_A\Big[\sqrt{f(\bw_t;z_i)}+\sqrt{f(\bw_t^{(i)};\tilde{z}_i)}\Big].
  \]
  Since $z_i$ and $\tilde{z}_i$ follow from the same distribution, we know
  \begin{equation}\label{symmetry}
    \ebb_{S,\widetilde{S},A}\big[\sqrt{f(\bw_t^{(i)};\tilde{z}_i)}\big]=\ebb_{S,A}\big[\sqrt{f(\bw_t;z_i)}\big].
  \end{equation}
  It then follows that
  \[
  \ebb_{S,\widetilde{S},A}\big[\|\bw_{t+1}-\bw_{t+1}^{(i)}\|_2\big]\leq \ebb_{S,\widetilde{S},A}\big[\|\bw_t-\bw_t^{(i)}\|\big]+\frac{2\sqrt{2L}\eta_t}{n}\ebb_{S,A}\big[\sqrt{f(\bw_t;z_i)}\big].
  \]
  Taking a summation of the above inequality and using $\bw_1=\bw_1^{(i)}$ then give \eqref{on-average-a}.

  We now turn to \eqref{on-average-b}. For the case $i_t=i$, it follows from \eqref{on-average-2} and the standard inequality $(a+b)^2\leq (1+p)a^2+(1+1/p)b^2$ that
  \begin{align}
    \|\bw_{t+1}-\bw_{t+1}^{(i)}\|_2^2 & \leq (1+p)\|\bw_{t}-\bw_{t}^{(i)}\|_2^2 + (1+p^{-1})\eta_t^2\|\partial f(\bw_t;z_i)-\partial f(\bw_t^{(i)};\tilde{z}_i)\|_2^2 \notag\\
     & \leq (1+p)\|\bw_{t}-\bw_{t}^{(i)}\|_2^2 +2(1+p^{-1})\eta_t^2\|\partial f(\bw_t;z_i)\|_2^2 + 2(1+p^{-1})\eta_t^2\|\partial f(\bw_t^{(i)};\tilde{z}_i)\|_2^2\notag\\
     & \leq (1+p)\|\bw_{t}-\bw_{t}^{(i)}\|_2^2 + 4(1+p^{-1})L\eta_t^2f(\bw_t;z_i) + 4(1+p^{-1})L\eta_t^2f(\bw_t^{(i)};\tilde{z}_i),\label{on-average-4}
  \end{align}
  where the last inequality is due to Lemma \ref{lem:self-bounding}.
  Combining \eqref{on-average-1}, the above inequality together and noticing the distribution of $i_t$, we derive
  \[
  \ebb_{A}\big[\|\bw_{t+1}-\bw_{t+1}^{(i)}\|_2^2\big]\leq (1+p/n)\ebb_A\big[\|\bw_{t}-\bw_{t}^{(i)}\|_2^2\big]+\frac{4(1+p^{-1})L\eta_t^2}{n}\ebb_A\big[f(\bw_t;z_i)+f(\bw_t^{(i)};\tilde{z}_i)\big].
  \]
  Analogous to \eqref{symmetry}, we have
  \begin{equation}\label{symmetry-2}
    \ebb_{S,\widetilde{S},A}\big[f(\bw_t^{(i)};\tilde{z}_i)\big]=\ebb_{S,A}\big[f(\bw_t;z_i)\big]
  \end{equation}
  and get
  \[
  \ebb_{S,\widetilde{S},A}\big[\|\bw_{t+1}-\bw_{t+1}^{(i)}\|_2^2\big]\leq (1+p/n)\ebb_{S,\widetilde{S},A}\big[\|\bw_{t}-\bw_{t}^{(i)}\|_2^2\big]+\frac{8(1+p^{-1})L\eta_t^2}{n}\ebb_{S,A}\big[f(\bw_t;z_i)\big].
  \]
  Multiplying both sides by $(1+p/n)^{-(t+1)}$ yields that
  \begin{multline*}
  (1+p/n)^{-(t+1)}\ebb_{S,\widetilde{S},A}\big[\|\bw_{t+1}-\bw_{t+1}^{(i)}\|_2^2\big]\leq (1+p/n)^{-t}\ebb_{S,\widetilde{S},A}\big[\|\bw_{t}-\bw_{t}^{(i)}\|_2^2\big]\\+\frac{8(1+p^{-1})L(1+p/n)^{-(t+1)}\eta_t^2}{n}\ebb_{S,A}\big[f(\bw_t;z_i)\big].
  \end{multline*}
  Taking a summation of the above inequality and using $\bw_1=\bw_1^{(i)}$, we get
  \[
  (1+p/n)^{-(t+1)}\ebb_{S,\widetilde{S},A}\big[\|\bw_{t+1}-\bw_{t+1}^{(i)}\|_2^2\big]
  \leq \sum_{j=1}^{t}\frac{8(1+p^{-1})L(1+p/n)^{-(j+1)}\eta_j^2}{n}\ebb_{S,A}\big[f(\bw_j;z_i)\big].
  \]
  The stated bound then follows.
\end{proof}

\begin{proof}[Proof of Theorem \ref{thm:on-average}]
  We first prove \eqref{on-average}. According to Lemma \ref{lem:on-average} (Eq. \eqref{on-average-a}), we know
  \[
  \ebb_{S,\widetilde{S},A}\Big[\frac{1}{n}\sum_{i=1}^{n}\|\bw_{t+1}-\bw_{t+1}^{(i)}\|_2\Big]\leq \frac{2\sqrt{2L}}{n^2}\sum_{i=1}^{n}\sum_{j=1}^{t}\eta_j\ebb_{S,A}\big[\sqrt{f(\bw_j;z_i)}\big].
  \]
  It then follows from the concavity of the square-root function and the Jensen's inequality that
  \begin{align*}
    \ebb_{S,\widetilde{S},A}\Big[\frac{1}{n}\sum_{i=1}^{n}\|\bw_{t+1}-\bw_{t+1}^{(i)}\|_2\Big]
    & \leq \frac{2\sqrt{2L}}{n}\sum_{j=1}^{t}\eta_j\ebb_{S,A}\bigg[\sqrt{n^{-1}\sum_{i=1}^{n}f(\bw_j;z_i)}\bigg] \\
    & = \frac{2\sqrt{2L}}{n}\sum_{j=1}^{t}\eta_j\ebb_{S,A}\Big[\sqrt{F_S(\bw_j)}\Big].
  \end{align*}
  This proves \eqref{on-average}.

  We now turn to \eqref{on-average-l2}.
  It follows from Lemma \ref{lem:on-average} (Eq. \eqref{on-average-b}) that
  \begin{align*}
    \ebb_{S,\widetilde{S},A}\Big[\frac{1}{n}\sum_{i=1}^{n}\|\bw_{t+1}-\bw_{t+1}^{(i)}\|_2^2\Big]
    & \leq \frac{8(1+p^{-1})L}{n^2}\sum_{i=1}^{n}\sum_{j=1}^{t}(1+p/n)^{t-j}\eta_j^2\ebb_{S,A}\big[f(\bw_j;z_i)\big]\\
    & = \frac{8(1+p^{-1})L}{n}\sum_{j=1}^{t}(1+p/n)^{t-j}\eta_j^2\ebb_{S,A}\big[F_S(\bw_j)\big].
  \end{align*}
  The proof is complete.
\end{proof}

\begin{proposition}[Stability bounds for non-convex learning]\label{thm:on-average-non-convex}
  Let Assumptions of Theorem \ref{thm:on-average} hold except that we do not require the convexity of $\bw\mapsto f(\bw;z)$. Then for any $p>0$ we have
  \[
    \ebb_{S,\widetilde{S},A}\Big[\frac{1}{n}\sum_{i=1}^{n}\|\bw_{t+1}-\bw_{t+1}^{(i)}\|_2^2\Big]
    \leq (1+p/n)(1+\eta_tL)^2\ebb_{S,\widetilde{S},A}\Big[\frac{1}{n}\sum_{i=1}^{n}\|\bw_{t}-\bw_{t}^{(i)}\|_2^2\Big]
    + \frac{8(1+p^{-1})L\eta_t^2}{n}\ebb_{S,A}\big[F_S(\bw_t)\big].
  \]
\end{proposition}
\begin{proof}
  If $i_t\neq i$, then by the $L$-smoothness of $f$ we know
  \begin{equation}\label{on-average-5}
  \|\bw_{t+1}-\bw_{t+1}^{(i)}\|_2\leq \|\bw_t-\bw_t^{(i)}\|_2+\eta_t\big\|\partial f(\bw_t;z_{i_t})-\partial f(\bw_t^{(i)},z_{i_t})\big\|_2\leq (1+\eta_tL)\|\bw_t-\bw_t^{(i)}\|_2.
  \end{equation}
  If $i_t= i$, then analogous to \eqref{on-average-4}, one can get
  \[
  \|\bw_{t+1}-\bw_{t+1}^{(i)}\|_2^2\leq (1+p)\|\bw_{t}-\bw_{t}^{(i)}\|_2^2 + 4(1+p^{-1})L\eta_t^2f(\bw_t;z_i) + 4(1+p^{-1})L\eta_t^2f(\bw_t^{(i)};\tilde{z}_i).
  \]
  By the uniform distribution of $i_t\in\{1,2,\ldots,n\}$ we can combine the above inequality and \eqref{on-average-5} to derive
  \[
  \ebb_A\big[\|\bw_{t+1}-\bw_{t+1}^{(i)}\|_2^2\big]\leq (1+p/n)(1+\eta_tL)^2\ebb_A\big[\|\bw_{t}-\bw_{t}^{(i)}\|_2^2\big] + \frac{4(1+p^{-1})L\eta_t^2}{n}\ebb_A\big[f(\bw_t;z_i)+f(\bw_t^{(i)};\tilde{z}_i)\big].
  \]
  This together with \eqref{symmetry-2} implies
  \[
  \ebb_{S,\widetilde{S},A}\big[\|\bw_{t+1}-\bw_{t+1}^{(i)}\|_2^2\big]\leq (1+p/n)(1+\eta_tL)^2\ebb_{S,\widetilde{S},A}\big[\|\bw_{t}-\bw_{t}^{(i)}\|_2^2\big] + \frac{8(1+p^{-1})L\eta_t^2}{n}\ebb_{S,A}\big[f(\bw_t;z_i)\big].
  \]
  It then follows that
  \begin{align*}
    & \ebb_{S,\widetilde{S},A}\Big[\frac{1}{n}\sum_{i=1}^{n}\|\bw_{t+1}-\bw_{t+1}^{(i)}\|_2^2\Big] \\
    & \leq (1+p/n)(1+\eta_tL)^2\ebb_{S,\widetilde{S},A}\Big[\frac{1}{n}\sum_{i=1}^{n}\|\bw_{t}-\bw_{t}^{(i)}\|_2^2\Big]
    + \frac{8(1+p^{-1})L\eta_t^2}{n^2}\sum_{i=1}^{n}\ebb_{S,A}\big[f(\bw_t;z_i)\big]\\
    & = (1+p/n)(1+\eta_tL)^2\ebb_{A}\Big[\frac{1}{n}\sum_{i=1}^{n}\|\bw_{t}-\bw_{t}^{(i)}\|_2^2\Big]
    + \frac{8(1+p^{-1})L\eta_t^2}{n}\ebb_{S,A}\big[F_S(\bw_t)\big].
  \end{align*}
  The proof is complete.
\end{proof}

\subsection{Generalization bounds\label{sec:proof-lip-gen}}

We now prove generalization bounds for SGD.
\begin{proof}[Proof of Theorem \ref{thm:gen-lipschitz}]
According to Part (c) of Lemma \ref{lem:computation-ave} with $\bw=\bw^*$, we know the following inequality
\begin{equation}\label{gen-lipschitz-1}
\sum_{t=1}^{T}\eta_t\ebb_A[F_S(\bw_t)-F_S(\bw^*)]  \leq (1/2+L\eta_1)\|\bw^*\|_2^2 + 2L\sum_{t=1}^{T}\eta_t^2F_S(\bw^*).
\end{equation}

Let $A(S)$ be the $(t+1)$-th iterate of SGD applied to the dataset $S$.
  We plug \eqref{on-average-l2} into Part (b) of Theorem \ref{thm:gen-model-stab}, and derive
  \[
  \ebb_{S,A}\big[F(\bw_{t+1})\big]\leq (1+L\gamma^{-1})\ebb_{S,A}\big[F_S(\bw_{t+1})\big] +
  \frac{4(1+p^{-1})(L+\gamma)L\big(1+p/n\big)^{t-1}}{n}\sum_{j=1}^{t}\eta_j^2\ebb_{S,A}\big[F_S(\bw_j)\big].
  \]
  We can plug \eqref{computation-ave-2} with $\bw=\bw^*$ into the above inequality, and derive
  \begin{multline*}
    \ebb_{S,A}\big[F(\bw_{t+1})\big]\leq (1+L\gamma^{-1})\ebb_{S,A}\big[F_S(\bw_{t+1})\big] \\+
  \frac{4(1+p^{-1})(L+\gamma)L\big(1+p/n\big)^{t-1}}{n}\big(\eta_1\|\bw^*\|_2^2+2\sum_{j=1}^{t}\eta_j^2\ebb_{S,A}[F_S(\bw^*)]\big).
\end{multline*}
We choose $p=n/T$, then $(1+p/n)^{T-1}=(1+1/T)^{T-1}\leq e$ and therefore the following inequality holds for all $t=1,\ldots,T$ (note $\ebb_{S,A}[F_S(\bw^*)]=F(\bw^*)$)
\[
\ebb_{S,A}\big[F(\bw_{t+1})\big]\leq (1+L\gamma^{-1})\ebb_{S,A}\big[F_S(\bw_{t+1})\big] +
  \frac{4(1+T/n)(L+\gamma)Le}{n}\big(\eta_1\|\bw^*\|_2^2+2\sum_{j=1}^{t}\eta_j^2F(\bw^*)\big).
\]
Multiplying both sides by $\eta_{t+1}$ followed with a summation gives
\[
  \sum_{t=1}^{T}\eta_t\ebb_{S,A}[F(\bw_{t})]\leq \big(1+L/\gamma\big)\sum_{t=1}^{T}\eta_t\ebb_{S,A}[F_S(\bw_{t})]+
  \frac{4(1+T/n)(L+\gamma)Le}{n}\sum_{t=1}^{T}\eta_t\big(\eta_1\|\bw^*\|_2^2+2\sum_{j=1}^{t-1}\eta_j^2F(\bw^*)\big).
\]

Putting \eqref{gen-lipschitz-1} into the above inequality then gives
\begin{multline*}
  \sum_{t=1}^{T}\eta_t\ebb_{S,A}[F(\bw_{t})]\leq \big(1+L/\gamma\big)\Big(\sum_{t=1}^{T}\eta_t\ebb_{S,A}[F_S(\bw^*)]+(1/2+L\eta_1)\|\bw^*\|_2^2 + 2L\sum_{t=1}^{T}\eta_t^2\ebb_{S,A}[F_S(\bw^*)]\Big)\\+
  \frac{4(1+T/n)(L+\gamma)Le}{n}\sum_{t=1}^{T}\eta_t\big(\eta_1\|\bw^*\|_2^2+2\sum_{j=1}^{t-1}\eta_j^2F(\bw^*)\big).
\end{multline*}
Since $\ebb_S[F_S(\bw^*)]=F(\bw^*)$, it follows that
\begin{multline*}
  \sum_{t=1}^{T}\eta_t\ebb_{S,A}[F(\bw_{t})-F(\bw^*)]\leq \frac{L}{\gamma}\sum_{t=1}^{T}\eta_tF(\bw^*)+
  \big(1+L/\gamma\big)\Big((1/2+L\eta_1)\|\bw^*\|_2^2 + 2L\sum_{t=1}^{T}\eta_t^2F(\bw^*)]\Big)\\+
  \frac{4(1+T/n)(L+\gamma)Le}{n}\sum_{t=1}^{T}\eta_t\big(\eta_1\|\bw^*\|_2^2+2\sum_{j=1}^{t-1}\eta_j^2F(\bw^*)\big).
\end{multline*}
The stated inequality then follows from Jensen's inequality. The proof is complete.
\end{proof}
\begin{proof}[Proof of Corollary \ref{cor:gen-lipschitz}]
  We first prove Part (a).
  For the chosen step size, we know
  \begin{equation}\label{cor-gen-lipschitz-1}
    \sum_{t=1}^{T}\eta_t^2=c^2\sum_{t=1}^{T}\frac{1}{T}=c^2\quad\text{and}\quad\sum_{t=1}^{T}\eta_t=c\sqrt{T}.
  \end{equation}
  The stated bound \eqref{cor-gen-lipschitz-a} then follows from Theorem \ref{thm:gen-lipschitz}, $\gamma=\sqrt{n}$ and \eqref{cor-gen-lipschitz-1}.

  We now prove Part (b). The stated bound \eqref{cor-gen-lipschitz-b} then follows from Theorem \ref{thm:gen-lipschitz}, $F(\bw^*)=0$ and $\gamma=1$.
  The proof is complete.
\end{proof}

\section{Proof on Learning without Bounded Gradients: Non-smooth Case}

\subsection{Stability bounds\label{sec:proof-holder-stab}}
Theorem \ref{thm:on-average-holder-maintext} is a direct application of the following general stability bounds with $p=n/t$. Therefore, it suffices to prove Theorem \ref{thm:on-average-holder}.
\begin{theorem}\label{thm:on-average-holder}
  Assume for all $z\in\zcal$, the function $\bw\mapsto f(\bw;z)$ is nonnegative, convex and $\partial f(\bw;z)$ is $(\alpha,L)$-H\"older continuous with $\alpha\in[0,1)$.
  Let $S,\widetilde{S}$ and $S^{(i)}$ be constructed as Definition \ref{def:aver-stab} and
  $c_{\alpha,3}=\frac{\sqrt{1-\alpha}}{\sqrt{1+\alpha}}(2^{-\alpha}L)^{\frac{1}{1-\alpha}}$. Let $\bw_t$ and $\bw_t^{(i)}$ be the $t$-th iterate produced by \eqref{SGD} based on $S$ and $S^{(i)}$, respectively. Then for any $p>0$ we have
    \begin{multline}
    \ebb_{S,\widetilde{S},A}\Big[\frac{1}{n}\sum_{i=1}^{n}\|\bw_{t+1}-\bw_{t+1}^{(i)}\|_2^2\Big]
      \leq c^2_{\alpha,3}\sum_{j=1}^{t}(1+p/n)^{t+1-j}\eta_j^{\frac{2}{1-\alpha}}
  \\+4(1+p^{-1})c^2_{\alpha,1}\sum_{j=1}^{t}\frac{(1+p/n)^{t-j}\eta_j^2}{n}
  \ebb_{S,A}\Big[F_S^{\frac{2\alpha}{1+\alpha}}(\bw_j)\Big].\label{on-average-holder-l2}
  \end{multline}
\end{theorem}
We require several lemmas to prove Theorem \ref{thm:on-average-holder}.
The following lemma establishes the co-coercivity of gradients for convex functions with H\"older continuous (sub)gradients. The case $\alpha=1$ can be found in \citet{nesterov2013introductory}. The case $\alpha\in(0,1)$ can be found in \citet{ying2017unregularized}. The case $\alpha=0$ follows directly from the convexity of $f$.
\begin{lemma}\label{lem:coercivity}
Assume for all $z\in\zcal$, the map $\bw\mapsto f(\bw;z)$ is convex,  and $\bw\mapsto \partial f(\bw;z)$ is $(\alpha,L)$-H\"older continuous with $\alpha\in[0,1]$. Then
for all $\bw,\tilde{\bw}$ we have
  \begin{equation}\label{coercivity}
    \big\langle\bw-\tilde{\bw},\partial f(\bw;z)-\partial f(\tilde{\bw};z)\big\rangle\geq \frac{2L^{-\frac{1}{\alpha}}\alpha}{1+\alpha}\|\partial f(\bw;z)-\partial f(\tilde{\bw};z)\|_2^{\frac{1+\alpha}{\alpha}}.
  \end{equation}
\end{lemma}

The following lemma controls the expansive behavior of the operator $\bw\mapsto\bw-\eta \partial f(\bw;z)$ for convex $f$ with H\"older continuous (sub)gradients.
\begin{lemma}\label{lem:holder-expansive}
Assume for all $z\in\zcal$, the map $\bw\mapsto f(\bw;z)$ is convex,  and $\bw\mapsto \partial f(\bw;z)$ is $(\alpha,L)$-H\"older continuous with $\alpha\in[0,1)$. Then for all $\bw\in\rbb^d$ and $\eta>0$ there holds
  \[
  \|\bw-\eta \partial f(\bw;z)-\tilde{\bw}+\eta \partial f(\tilde{\bw};z)\|_2^2 \leq \|\bw-\tilde{\bw}\|_2^2+c^2_{\alpha,3}\eta^{\frac{2}{1-\alpha}}.
  \]
\end{lemma}
\begin{proof}
  The following equality holds
  \begin{equation}
  \|\bw-\eta \partial f(\bw;z)-\tilde{\bw}+\eta \partial f(\tilde{\bw};z)\|_2^2=\|\bw-\tilde{\bw}\|_2^2
  +\eta^2\|\partial f(\bw;z)-\partial f(\tilde{\bw};z)\|_2^2-2\eta\langle\bw-\tilde{\bw},\partial f(\bw;z)-\partial f(\tilde{\bw};z)\rangle.\label{expansive-1}
  \end{equation}
  We first consider the case $\alpha=0$. In this case, it follows from Definition \ref{def:holder} and Lemma \ref{lem:coercivity} with $\alpha=0$ that
  \[
  \|\bw-\eta \partial f(\bw;z)-\tilde{\bw}+\eta \partial f(\tilde{\bw};z)\|_2^2\leq\|\bw-\tilde{\bw}\|_2^2
  +\eta^2L^2.
  \]
  We now consider the case $\alpha>0$.
  According to Lemma \ref{lem:coercivity}, we know
  \begin{align*}
     & \|\partial f(\bw;z)-\partial f(\tilde{\bw};z)\|_2^2  \leq \Big(\frac{L^{\frac{1}{\alpha}}(1+\alpha)}{2\alpha}\big\langle\bw-\tilde{\bw},\partial f(\bw;z)-\partial f(\tilde{\bw};z)\big\rangle\Big)^{\frac{2\alpha}{1+\alpha}} \\
     & =\Big(\frac{1+\alpha}{\eta\alpha}\big\langle\bw-\tilde{\bw},\partial f(\bw;z)-\partial f(\tilde{\bw};z)\big\rangle\Big)^{\frac{2\alpha}{1+\alpha}}\Big(\eta^{\frac{2\alpha}{1+\alpha}}L^{\frac{2}{1+\alpha}}2^{-\frac{2\alpha}{1+\alpha}}\Big) \\
     & \leq \frac{2\alpha}{1+\alpha}\Big(((1+\alpha)/(\eta\alpha))^{\frac{2\alpha}{1+\alpha}}\big\langle\bw-\tilde{\bw},\partial f(\bw;z)-\partial f(\tilde{\bw};z)\big\rangle^{\frac{2\alpha}{1+\alpha}}\Big)^{\frac{1+\alpha}{2\alpha}}+\frac{1-\alpha}{1+\alpha}\Big(\eta^{\frac{2\alpha}{1+\alpha}}L^{\frac{2}{1+\alpha}}2^{-\frac{2\alpha}{1+\alpha}}\Big)^{\frac{1+\alpha}{1-\alpha}} \\
    & =2\eta^{-1}\big\langle\bw-\tilde{\bw},\partial f(\bw;z)-\partial f(\tilde{\bw};z)\big\rangle+\frac{1-\alpha}{1+\alpha}\eta^{\frac{2\alpha}{1-\alpha}}(2^{-\alpha}L)^{\frac{2}{1-\alpha}},
  \end{align*}
  where we have used Young's inequality~\eqref{young}.
  Plugging the above inequality back into \eqref{expansive-1}, we derive
  \[
    \|\bw-\eta \partial f(\bw;z)-\tilde{\bw}+\eta \partial f(\tilde{\bw};z)\|_2^2 \leq \|\bw-\tilde{\bw}\|_2^2+\frac{1-\alpha}{1+\alpha}\eta^{2+\frac{2\alpha}{1-\alpha}}(2^{-\alpha}L)^{\frac{2}{1-\alpha}}.
  \]
  Combining the above two cases together, we get the stated bound with the definition of $c_{\alpha,3}$ given in Theorem \ref{thm:on-average-holder}.
  The proof is complete.
\end{proof}

\begin{proof}[Proof of Theorem \ref{thm:on-average-holder}]
  For the case $i_t\neq i$, it follows from Lemma \ref{lem:holder-expansive} that
  \[
  \|\bw_{t+1}-\bw_{t+1}^{(i)}\|_2^2\leq \|\bw_{t}-\eta_t\partial f(\bw_t;z_{i_t})-\bw_{t}^{(i)}+\eta_t\partial f(\bw_t^{(i)},z_{i_t})\|_2^2\leq \|\bw_{t}-\bw_{t}^{(i)}\|_2^2+c^2_{\alpha,3}\eta_t^{\frac{2}{1-\alpha}}.
  \]
  If $i_t=i$, by \eqref{on-average-2} and the standard inequality $(a+b)^2\leq (1+p)a^2+(1+1/p)b^2$, we get
  \begin{align*}
    \|\bw_{t+1}-\bw_{t+1}^{(i)}\|_2^2
    & \leq (1+p)\|\bw_{t}-\bw_{t}^{(i)}\|_2^2+
    2(1+p^{-1})\eta_t^2\big(\|\partial f(\bw_t;z_i)\big\|_2^2+\big\|\partial f(\bw_t^{(i)};\tilde{z}_i)\|_2^2\big).
  \end{align*}
  Combining the above two inequalities together, using the self-bounding property (Lemma \ref{lem:self-bounding}) and noticing the distribution of $i_t$, we derive
  \begin{multline*}
  \ebb_{A}\big[\|\bw_{t+1}-\bw_{t+1}^{(i)}\|_2^2\big]\leq (1+p/n)\Big(\ebb_A\big[\|\bw_{t}-\bw_{t}^{(i)}\|_2^2\big]+c^2_{\alpha,3}\eta_t^{\frac{2}{1-\alpha}}\Big)
  +\\
  \frac{2(1+p^{-1})c^2_{\alpha,1}\eta_t^2}{n}\ebb_A\big[f^{\frac{2\alpha}{1+\alpha}}(\bw_t;z_i)+f^{\frac{2\alpha}{1+\alpha}}(\bw_t^{(i)};\tilde{z}_i)\big].
  \end{multline*}
  Analogous to \eqref{symmetry}, we know
  \[
  \ebb_{S,\widetilde{S},A}\big[f^{\frac{2\alpha}{1+\alpha}}(\bw_t^{(i)};\tilde{z}_i)\big]=\ebb_{S,A}\big[f^{\frac{2\alpha}{1+\alpha}}(\bw_t;z_i)\big]
  \]
  and therefore
  \begin{multline*}
  \ebb_{S,\widetilde{S},A}\big[\|\bw_{t+1}-\bw_{t+1}^{(i)}\|_2^2\big]\leq (1+p/n)\Big(\ebb_{S,\widetilde{S},A}\big[\|\bw_{t}-\bw_{t}^{(i)}\|_2^2\big]+c^2_{\alpha,3}\eta_t^{\frac{2}{1-\alpha}}\Big)
  +\\
  \frac{4(1+p^{-1})c^2_{\alpha,1}\eta_t^2}{n}\ebb_{S,A}\big[f^{\frac{2\alpha}{1+\alpha}}(\bw_t;z_i)\big].
  \end{multline*}
  Multiplying both sides by $(1+p/n)^{-(t+1)}$ gives
  \begin{multline*}
  (1+p/n)^{-(t+1)}\ebb_{S,\widetilde{S},A}\big[\|\bw_{t+1}-\bw_{t+1}^{(i)}\|_2^2\big]\leq (1+p/n)^{-t}\Big(\ebb_{S,\widetilde{S},A}\big[\|\bw_{t}-\bw_{t}^{(i)}\|_2^2\big]+c^2_{\alpha,3}\eta_t^{\frac{2}{1-\alpha}}\Big)
  +\\
  \frac{4(1+p^{-1})c^2_{\alpha,1}(1+p/n)^{-(t+1)}\eta_t^2}{n}\ebb_{S,A}\big[f^{\frac{2\alpha}{1+\alpha}}(\bw_t;z_i)\big].
  \end{multline*}
  Taking a summation of the above inequality and using $\bw_1=\bw_1^{(i)}$, we derive
  \begin{multline*}
  (1+p/n)^{-(t+1)}\ebb_{S,\widetilde{S},A}\big[\|\bw_{t+1}-\bw_{t+1}^{(i)}\|_2^2\big]
  \leq c^2_{\alpha,3}\sum_{j=1}^{t}(1+p/n)^{-j}\eta_j^{\frac{2}{1-\alpha}}\\
  +\frac{4(1+p^{-1})c^2_{\alpha,1}}{n}\sum_{j=1}^{t}(1+p/n)^{-(j+1)}\eta_j^2\ebb_{S,A}\big[f^{\frac{2\alpha}{1+\alpha}}(\bw_j;z_i)\big].
  \end{multline*}
  We can take an average over $i$ and get
  \begin{multline*}
  \frac{1}{n}\sum_{i=1}^{n}\ebb_{S,\widetilde{S},A}\big[\|\bw_{t+1}-\bw_{t+1}^{(i)}\|_2^2\big]
  \leq c^2_{\alpha,3}\sum_{j=1}^{t}(1+p/n)^{t+1-j}\eta_j^{\frac{2}{1-\alpha}}\\
  +\frac{4(1+p^{-1})c^2_{\alpha,1}}{n^2}\sum_{i=1}^{n}\sum_{j=1}^{t}(1+p/n)^{t-j}\eta_j^2\ebb_{S,A}\big[f^{\frac{2\alpha}{1+\alpha}}(\bw_j;z_i)\big].
  \end{multline*}
  It then follows from the concavity of the function $x\mapsto x^{\frac{2\alpha}{1+\alpha}}$ and the Jensen's inequality that
  \begin{multline*}
    \frac{1}{n}\sum_{i=1}^{n}\ebb_{S,\widetilde{S},A}\big[\|\bw_{t+1}-\bw_{t+1}^{(i)}\|_2^2\big]
  \leq c^2_{\alpha,3}\sum_{j=1}^{t}(1+p/n)^{t+1-j}\eta_j^{\frac{2}{1-\alpha}}+\\
  4(1+p^{-1})c^2_{\alpha,1}\sum_{j=1}^{t}\frac{(1+p/n)^{t-j}\eta_j^2}{n}\ebb_{S,A}\Big[\Big(\frac{1}{n}\sum_{i=1}^nf(\bw_j;z_i)\Big)^{\frac{2\alpha}{1+\alpha}}\Big].
  \end{multline*}
  The stated inequality then follows from the definition of $F_S$. The proof is complete.
\end{proof}

\subsection{Generalization errors\label{sec:proof-holder-gen}}
Theorem \ref{thm:error-holder-cor-a} can be considered as an instantiation of the following proposition on generalization error bounds with specific choices of $\gamma,T$ and $\theta$. In this subsection, we first give the proof of Theorem \ref{thm:error-holder-cor-a} based on Proposition \ref{prop:gen-holder-up}, and then turn to the proof of Proposition \ref{prop:gen-holder-up}.
\begin{proposition}\label{prop:gen-holder-up}
  Assume for all $z\in\zcal$, the function $\bw\mapsto f(\bw;z)$ is nonnegative, convex, and $\partial f(\bw;z)$ is $(\alpha,L)$-H\"older continuous with $\alpha\in[0,1)$. Let $\{\bw_t\}_t$ be produced by \eqref{SGD} with step sizes $\eta_t=cT^{-\theta},\theta\in[0,1]$ satisfying
  $\theta\geq(1-\alpha)/2$. Then for all $T$ satisfying $n=O(T)$ and any $\gamma>0$ we have
  \begin{multline*}
    \ebb_{S,A}[F(\bw_T^{(1)})]-F(\bw^*)
   = O\Big(\gamma^{\frac{1+\alpha}{\alpha-1}}\Big)
   +O(\gamma)\Big(T^{1-\frac{2\theta}{1-\alpha}}+n^{-2}T^{\frac{2-2\theta}{1+\alpha}}+n^{-2}T^{2-2\theta}F^{\frac{2\alpha}{1+\alpha}}(\bw^*)\Big)\\
   +O(1/\gamma)\Big(T^{-\frac{2\alpha(1-\theta)}{1+\alpha}}+F^{\frac{2\alpha}{1+\alpha}}(\bw^*)\Big)
  +O(T^{\theta-1})+O(T^{\frac{\alpha\theta-\theta-2\alpha}{1+\alpha}})+O(T^{-\theta}F^{\frac{2\alpha}{1+\alpha}}(\bw^*)).
  \end{multline*}
\end{proposition}
\begin{proof}[Proof of Theorem \ref{thm:error-holder-cor-a}]
It can be checked that $\theta$ considered in Parts (a)-(c) satisfy $\theta\geq(1-\alpha)/2$. Therefore Proposition \ref{prop:gen-holder-up} holds.
If $\gamma=\sqrt{n}$, then by Proposition \ref{prop:gen-holder-up} we know
\begin{equation}\label{error-holder-up-cor-1}
\ebb_{S,A}[F(\bw_T^{(1)})]-F(\bw^*)=O(n^{\frac{1}{2}}T^{1-\frac{2\theta}{1-\alpha}})
+O(n^{-\frac{3}{2}}T^{2-2\theta})+O(n^{-\frac{1}{2}})+O(T^{\theta-1})+O(T^{\frac{\alpha\theta-\theta-2\alpha}{1+\alpha}})+O(T^{-\theta}).
\end{equation}

We first prove Part (a). Since $\alpha\geq1/2$, $\theta=1/2$ and $T\asymp n$, it follows from \eqref{error-holder-up-cor-1} that
\[
\ebb_{S,A}[F(\bw_T^{(1)})]-F(\bw^*)=O(n^{-\frac{1}{2}})+O(n^{\frac{3}{2}-\frac{1}{1-\alpha}})+O(n^{-\frac{3\alpha+1}{2(1+\alpha)}})=O(n^{-\frac{1}{2}}),
\]
where we have used $\frac{3}{2}-\frac{1}{1-\alpha}\leq-\frac{1}{2}$ due to $\alpha\geq1/2$. This shows Part (a).

We now prove Part (b).
Since $\alpha<1/2$, $T\asymp n^{\frac{2-\alpha}{1+\alpha}}$ and $\theta=\frac{3-3\alpha}{2(2-\alpha)}\geq1/2$ in \eqref{error-holder-up-cor-1}, the following inequalities hold
\begin{align*}
  &\sqrt{n}T^{1-\frac{2\theta}{1-\alpha}}\asymp\sqrt{n}T^{1-\frac{3-3\alpha}{(1-\alpha)(2-\alpha)}}\asymp\sqrt{n}n^{-\frac{(1+\alpha)(2-\alpha)}{(2-\alpha)(1+\alpha)}}\asymp n^{-\frac{1}{2}} \\
  &n^{-\frac{3}{2}}T^{2-2\theta}\asymp n^{-\frac{3}{2}}T^{\frac{4-2\alpha-3+3\alpha}{2-\alpha}}\asymp n^{-\frac{3}{2}}T^{\frac{1+\alpha}{2-\alpha}}\asymp n^{-\frac{1}{2}} \\
  &T^{\theta-1}\asymp T^{\frac{3-3\alpha-4+2\alpha}{2(2-\alpha)}}\asymp T^{-\frac{1+\alpha}{2(2-\alpha)}}\asymp n^{-\frac{1}{2}}\\
  &T^{-\theta}\asymp n^{\frac{2-\alpha}{1+\alpha}\frac{3\alpha-3}{2(2-\alpha)}}\asymp n^{\frac{3\alpha-3}{2(1+\alpha)}}=O(n^{-\frac{1}{2}}),
\end{align*}
where we have used $(3\alpha-3)/(1+\alpha)\leq-1$ due to $\alpha<1/2$.
Furthermore, since $\theta\geq1/2$ and $\alpha<1/2$ we know $\frac{\alpha\theta-\theta-2\alpha}{1+\alpha}\leq -\frac{1}{2}$ and $T\asymp n^{\frac{2-\alpha}{1+\alpha}}\geq n$. Therefore $T^{\frac{\alpha\theta-\theta-2\alpha}{1+\alpha}}=O(T^{-\frac{1}{2}})=O(n^{-\frac{1}{2}})$. Plugging the above inequalities into \eqref{error-holder-up-cor-1} gives the stated bound in Part (b).

We now turn to Part (c).
Since $F(\bw^*)=0$, Proposition \ref{prop:gen-holder-up} reduces to
\[
    \ebb_{S,A}[F(\bw_T^{(1)})]-F(\bw^*)
   = O\Big(\gamma^{\frac{1+\alpha}{\alpha-1}}\Big)
   +O(\gamma)\Big(T^{1-\frac{2\theta}{1-\alpha}}+n^{-2}T^{\frac{2-2\theta}{1+\alpha}}\Big)
   +O\Big(\gamma^{-1}T^{-\frac{2\alpha(1-\theta)}{1+\alpha}}\Big)
  +O(T^{\theta-1})+O(T^{\frac{\alpha\theta-\theta-2\alpha}{1+\alpha}}).
\]
With $\gamma=nT^{\theta-1}$, we further get
\begin{equation}\label{error-holder-up-cor-2}
    \ebb_{S,A}[F(\bw_T^{(1)})]-F(\bw^*)
   = O\Big(\big(n^{-1}T^{1-\theta}\big)^{\frac{1+\alpha}{1-\alpha}}\Big)
   +O(nT^{-\frac{(1+\alpha)\theta}{1-\alpha}})+O(n^{-1}T^{\frac{(\theta-1)(\alpha-1)}{1+\alpha}})+O(T^{\theta-1})+O(T^{\frac{\alpha\theta-\theta-2\alpha}{1+\alpha}}).
\end{equation}
For the choice $T=n^{\frac{2}{1+\alpha}}$ and $\theta=\frac{3-\alpha^2-2\alpha}{4}$, we know $1-\theta=(1+\alpha)^2/4$ and therefore
\begin{align*}
  & \Big(n^{-1}T^{1-\theta}\Big)^{\frac{1+\alpha}{1-\alpha}}\asymp \Big(n^{-1}n^{\frac{2}{1+\alpha}\frac{(1+\alpha)^2}{4}}\Big)^{\frac{1+\alpha}{1-\alpha}}= n^{-\frac{1+\alpha}{2}}\\
  & nT^{-\frac{(1+\alpha)\theta}{1-\alpha}}\asymp n^{1-\frac{2\theta}{1-\alpha}}=n^{1+\frac{2\alpha+\alpha^2-3}{2(1-\alpha)}}= n^{\frac{\alpha^2-1}{2-2\alpha}}= n^{-\frac{1+\alpha}{2}} \\
  & n^{-1}T^{\frac{(\theta-1)(\alpha-1)}{1+\alpha}}\asymp n^{-1}T^{\frac{(1-\alpha)(1+\alpha)^2}{4(1+\alpha)}}\asymp
  n^{-1}T^{\frac{(1-\alpha)(1+\alpha)}{4}}\asymp n^{-\frac{1+\alpha}{2}}\\
  & T^{\theta-1}\asymp n^{-\frac{2}{1+\alpha}\frac{(1+\alpha)^2}{4}}\asymp n^{-\frac{1+\alpha}{2}}.
\end{align*}
Furthermore,
\[
(\theta-1)(1+\alpha)-(\alpha\theta-\theta-2\alpha)=2\theta+\alpha-1=2^{-1}\big(3-\alpha^2-2\alpha+2\alpha-2\big)\geq0
\]
and therefore
\[
T^{\theta-1}\geq T^{\frac{\alpha\theta-\theta-2\alpha}{1+\alpha}}.
\]
Plugging the above inequalities into \eqref{error-holder-up-cor-2} gives the stated bound in Part (c). The proof is complete.
\end{proof}

To  prove Proposition \ref{prop:gen-holder-up}, we first introduce an useful lemma to address some involved series.
\begin{lemma}\label{lem:big-O}
  Assume for all $z\in\zcal$, the function $\bw\mapsto f(\bw;z)$ is nonnegative, convex, and $\partial f(\bw;z)$ is $(\alpha,L)$-H\"older continuous with $\alpha\in[0,1)$. Let $\{\bw_t\}_t$ be produced by \eqref{SGD} with step sizes $\eta_t=cT^{-\theta},\theta\in[0,1]$ satisfying
  $\theta\geq\frac{1-\alpha}{2}$. Then
  \begin{gather}
  \Big(\sum_{t=1}^{T}\eta_t^2\Big)^{\frac{1-\alpha}{1+\alpha}}\Big(\eta_1\|\bw^*\|_2^2+2\sum_{t=1}^{T}\eta_t^2F(\bw^*)+c_{\alpha,2}\sum_{t=1}^{T}\eta_t^{\frac{3-\alpha}{1-\alpha}}\Big)^{\frac{2\alpha}{1+\alpha}}=
  O(T^{\frac{1-\alpha-2\theta}{1+\alpha}})+O(T^{1-2\theta}F^{\frac{2\alpha}{1+\alpha}}(\bw^*)),\label{big-O-a}\\
  \sum_{t=1}^{T}\eta_t^2\big(\ebb_{S,A}[F_S(\bw_t)]\big)^{\frac{2\alpha}{1+\alpha}}=
  O(T^{\frac{1-\alpha-2\theta}{1+\alpha}})+O(T^{1-2\theta}F^{\frac{2\alpha}{1+\alpha}}(\bw^*)),\label{big-O-b}\\
  \sum_{t=1}^{T}\eta_t\big(\ebb_{S,A}[F_S(\bw_{t})]\big)^{\frac{2\alpha}{1+\alpha}}=
  O(T^{\frac{(1-\alpha)(1-\theta)}{1+\alpha}})+O(T^{1-\theta}F^{\frac{2\alpha}{1+\alpha}}(\bw^*)).\label{big-O-c}
  \end{gather}
\end{lemma}
\begin{proof}
We first prove \eqref{big-O-a}.
For the step size sequence $\eta_t=cT^{-\theta}$, we have
\begin{align*}
  &\Big(\sum_{t=1}^{T}\eta_t^2\Big)^{\frac{1-\alpha}{1+\alpha}}\Big(\eta_1\|\bw^*\|_2^2+2\sum_{t=1}^{T}\eta_t^2F(\bw^*)+c_{\alpha,2}\sum_{t=1}^{T}\eta_t^{\frac{3-\alpha}{1-\alpha}}\Big)^{\frac{2\alpha}{1+\alpha}}\notag\\
  &=O(T^{\frac{(1-2\theta)(1-\alpha)}{1+\alpha}})\Big(T^{-\theta}+T^{1-2\theta}F(\bw^*)+T^{1-\frac{(3-\alpha)\theta}{1-\alpha}}\Big)^{\frac{2\alpha}{1+\alpha}}\\
  &= O(T^{\frac{1-\alpha-2\theta}{1+\alpha}})+O(T^{1-2\theta}F^{\frac{2\alpha}{1+\alpha}}(\bw^*))+O(T^{\frac{1-\alpha-2\theta}{1-\alpha}})\\
  &= O(T^{\frac{1-\alpha-2\theta}{1+\alpha}})+O(T^{1-2\theta}F^{\frac{2\alpha}{1+\alpha}}(\bw^*)),
\end{align*}
where we have used the subadditivity of $x\mapsto x^{\frac{2\alpha}{1+\alpha}}$, the identity
\[
\frac{(1-2\theta)(1-\alpha)}{1+\alpha}+\frac{2\alpha}{1+\alpha}\frac{1-\alpha-3\theta+\alpha\theta}{1-\alpha}=\frac{1-\alpha-2\theta}{1-\alpha}
\]
in the second step and $\theta\geq(1-\alpha)/2$ in the third step (the third term is dominated by the first term).
This shows \eqref{big-O-a}.

We now consider \eqref{big-O-b}.
Taking an expectation over both sides of \eqref{computation-ave-4} with $\bw=\bw^*$, we get
\[
\sum_{t=1}^{T}\eta_t^2\ebb_{S,A}[F_S(\bw_t)]\leq\eta_1\|\bw^*\|_2^2+2\sum_{t=1}^{T}\eta_t^2\ebb_S[F_S(\bw^*)]+c_{\alpha,2}\sum_{t=1}^{T}\eta_t^{\frac{3-\alpha}{1-\alpha}}.
\]
According to the Jensen's inequality and the concavity of $x\mapsto x^{\frac{2\alpha}{1+\alpha}}$, we know
\begin{align*}
     \sum_{t=1}^{T}\eta_t^2\big(\ebb_{S,A}[F_S(\bw_t)]\big)^{\frac{2\alpha}{1+\alpha}} & \leq \sum_{t=1}^{T}\eta_t^2\Big(\frac{\sum_{t=1}^{T}\eta_t^2\ebb_{S,A}[F_S(\bw_t)]}{\sum_{t=1}^{T}\eta_t^2}\Big)^{\frac{2\alpha}{1+\alpha}} \\
      & \leq \Big(\sum_{t=1}^{T}\eta_t^2\Big)^{\frac{1-\alpha}{1+\alpha}}\Big(\eta_1\|\bw^*\|_2^2+2\sum_{t=1}^{T}\eta_t^2F(\bw^*)+c_{\alpha,2}\sum_{t=1}^{T}\eta_t^{\frac{3-\alpha}{1-\alpha}}\Big)^{\frac{2\alpha}{1+\alpha}}\\
      &= O(T^{\frac{1-\alpha-2\theta}{1+\alpha}})+O(T^{1-2\theta}F^{\frac{2\alpha}{1+\alpha}}(\bw^*)),
\end{align*}
where we have used \eqref{big-O-a} in the last step. This shows \eqref{big-O-b}.

Finally, we show \eqref{big-O-c}. Since we consider step sizes $\eta_t=cT^{-\theta}$, it follows from \eqref{big-O-b} that
\begin{align*}
  \sum_{t=1}^{T}\eta_t\big(\ebb_{S,A}[F_S(\bw_{t})]\big)^{\frac{2\alpha}{1+\alpha}} & = \big(cT^{-\theta}\big)^{-1}\sum_{t=1}^{T}\eta_t^2\big(\ebb_{S,A}[F_S(\bw_{t})]\big)^{\frac{2\alpha}{1+\alpha}}\\
  & = O(T^{\frac{1-\alpha-2\theta}{1+\alpha}+\theta})+O(T^{1-\theta}F^{\frac{2\alpha}{1+\alpha}}(\bw^*)).
\end{align*}
This proves \eqref{big-O-c} and finishes the proof.
\end{proof}
\begin{proof}[Proof of Proposition \ref{prop:gen-holder-up}]
Since $\ebb_S[F_S(\bw^*)]=F(\bw^*)$, we can decompose the excess generalization error into an estimation error and an optimization error as follows
\begin{multline}\label{gen-holder-up-0}
  \Big(\sum_{t=1}^{T}\eta_t\Big)^{-1}\sum_{t=1}^{T}\eta_t\big(\ebb_{S,A}[F(\bw_t)]-F(\bw^*)\big)
  =\Big(\sum_{t=1}^{T}\eta_t\Big)^{-1}\sum_{t=1}^{T}\eta_t\ebb_{S,A}[F(\bw_t)-F_S(\bw_t)]\\
  +\Big(\sum_{t=1}^{T}\eta_t\Big)^{-1}\sum_{t=1}^{T}\eta_t\ebb_{S,A}[F_S(\bw_t)-F_S(\bw^*)].
\end{multline}
Our idea is to address separately the above estimation error and optimization error.

We first address \textbf{estimation errors}.
Plugging \eqref{on-average-holder-l2} back into Theorem \ref{thm:gen-model-stab} (Part (c)) with $A(S)=\bw_{t+1}$, we derive
\begin{multline*}
  \ebb_{S,A}\big[F(\bw_{t+1})-F_S(\bw_{t+1})\big]\leq  \frac{c^2_{\alpha,1}}{2\gamma}\ebb_{S,A}\Big[F^{\frac{2\alpha}{1+\alpha}}(\bw_{t+1})\Big]+
  2^{-1}\gamma c^2_{\alpha,3}\sum_{j=1}^{t}(1+p/n)^{t+1-j}\eta_j^{\frac{2}{1-\alpha}}\\
  +2\gamma(1+p^{-1})c^2_{\alpha,1}\sum_{j=1}^{t}\frac{(1+p/n)^{t-j}\eta_j^2}{n}
  \ebb_{S,A}\Big[F_S^{\frac{2\alpha}{1+\alpha}}(\bw_j)\Big].
\end{multline*}
By the concavity and sub-additivity of $x\mapsto x^{\frac{2\alpha}{1+\alpha}}$, we know
\begin{align*}
\ebb_{S,A}\big[F^{\frac{2\alpha}{1+\alpha}}(\bw_{t+1})\big]&\leq \big(\ebb_{S,A}[F(\bw_{t+1})]-\ebb_{S,A}[F_S(\bw_{t+1})]+\ebb_{S,A}[F_S(\bw_{t+1})]\big)^{\frac{2\alpha}{1+\alpha}}\\
&\leq \delta_{t+1}^{\frac{2\alpha}{1+\alpha}}
+\big(\ebb_{S,A}[F_S(\bw_{t+1})]\big)^{\frac{2\alpha}{1+\alpha}},
\end{align*}
where we denote $\delta_j=\max\{\ebb_{S,A}[F(\bw_j)]-\ebb_{S,A}[F_S(\bw_j)],0\}$ for all $j\in\nbb$.
It then follows from $p=n/T$ that
\begin{multline*}
  \delta_{t+1}\leq  \frac{c^2_{\alpha,1}}{2\gamma}\Big(\delta_{t+1}^{\frac{2\alpha}{1+\alpha}}+\big(\ebb_{S,A}[F_S(\bw_{t+1})]\big)^{\frac{2\alpha}{1+\alpha}}\Big)
  +2^{-1}\gamma ec^2_{\alpha,3}\sum_{j=1}^{t}\eta_j^{\frac{2}{1-\alpha}}\\
  +\frac{2e\gamma(1+T/n)c^2_{\alpha,1}}{n}\sum_{j=1}^{t}\eta_j^2\big(\ebb_{S,A}[F_S(\bw_j)]\big)^{\frac{2\alpha}{1+\alpha}}.
\end{multline*}
Solving the above inequality of $\delta_{t+1}$ gives the following inequality for all $t\leq T$
\begin{multline*}
\delta_{t+1}=O\Big(\gamma^{\frac{1+\alpha}{\alpha-1}}\Big)+O\Big(\gamma^{-1}\big(\ebb_{S,A}[F_S(\bw_{t+1})]\big)^{\frac{2\alpha}{1+\alpha}}\Big)\\
+O\Big(\gamma\sum_{j=1}^{t}\eta_j^{\frac{2}{1-\alpha}}\Big)+O\Big(\gamma(n^{-1}+Tn^{-2})\sum_{j=1}^{t}\eta_j^2\big(\ebb_{S,A}[F_S(\bw_j)]\big)^{\frac{2\alpha}{1+\alpha}}\Big).
\end{multline*}
It then follows from the definition of $\delta_{t}$ that (note $n=O(T)$)
\begin{multline*}
  \Big(\sum_{t=1}^{T}\eta_t\Big)^{-1}\sum_{t=1}^{T}\eta_t\big(\ebb_{S,A}[F(\bw_t)]-\ebb_{S,A}[F_S(\bw_t)]\big) =  O\Big(\gamma^{\frac{1+\alpha}{\alpha-1}}\Big)
  + O\Big(\gamma\sum_{t=1}^{T}\eta_t^{\frac{2}{1-\alpha}}\Big)\\
  +O\Big(\gamma^{-1}\Big(\sum_{t=1}^{T}\eta_t\Big)^{-1}\sum_{t=1}^{T}\eta_t\big(\ebb_{S,A}[F_S(\bw_{t})]\big)^{\frac{2\alpha}{1+\alpha}}\Big)
  +O\Big(\gamma Tn^{-2}\sum_{t=1}^{T}\eta_t^2\big(\ebb_{S,A}[F_S(\bw_t)]\big)^{\frac{2\alpha}{1+\alpha}}\Big).
\end{multline*}
By $\eta_t=cT^{-\theta}$, \eqref{big-O-b} and \eqref{big-O-c}, we further get
\begin{multline}\label{gen-holder-up-2}
  \Big(\sum_{t=1}^{T}\eta_t\Big)^{-1}\sum_{t=1}^{T}\eta_t\big(\ebb_{S,A}[F(\bw_t)]-\ebb_{S,A}[F_S(\bw_t)]\big) =  O\Big(\gamma^{\frac{1+\alpha}{\alpha-1}}\Big)
  +O\Big(\gamma T^{1-\frac{2\theta}{1-\alpha}}\Big)\\
  +O\Big(\gamma^{-1}T^{\theta-1}\big(T^{\frac{(1-\alpha)(1-\theta)}{1+\alpha}}+T^{1-\theta}F^{\frac{2\alpha}{1+\alpha}}(\bw^*)\big)\Big)
  +O\Big(\gamma Tn^{-2}\big(T^{\frac{1-\alpha-2\theta}{1+\alpha}}+T^{1-2\theta}F^{\frac{2\alpha}{1+\alpha}}(\bw^*)\big)\Big).
\end{multline}

We now consider \textbf{optimization errors}. By Lemma \ref{lem:computation-ave} (Part (d) with $\bw=\bw^*$) and the concavity of $x\mapsto x^{\frac{2\alpha}{1+\alpha}}$, we know
\begin{align*}
&\Big(\sum_{t=1}^{T}\eta_t\Big)^{-1}\sum_{t=1}^{T}\eta_t\ebb_{S,A}[F_S(\bw_t)-F_S(\bw^*)]\\
&\leq \Big(2\sum_{t=1}^{T}\eta_t\Big)^{-1}\|\bw^*\|_2^2+
        c_{\alpha,1}^2\Big(2\sum_{t=1}^{T}\eta_t\Big)^{-1}\Big(\sum_{t=1}^{T}\eta_t^2\Big)^{\frac{1-\alpha}{1+\alpha}}
   \Big(\eta_1\|\bw^*\|_2^2+2\sum_{t=1}^{T}\eta_t^2F(\bw^*)+c_{\alpha,2}\sum_{t=1}^{T}\eta_t^{\frac{3-\alpha}{1-\alpha}}\Big)^{\frac{2\alpha}{1+\alpha}}\\
& = O(T^{\theta-1})+O(T^{\frac{1-\alpha-2\theta}{1+\alpha}+\theta-1})+O(T^{-\theta}F^{\frac{2\alpha}{1+\alpha}}(\bw^*)),
\end{align*}
where we have used \eqref{big-O-a} in the last step.

Plugging the above optimization error bound and the estimation error bound \eqref{gen-holder-up-2} back into the error decomposition \eqref{gen-holder-up-0}, we finally derive the following generalization error bounds
\begin{multline*}
  \Big(\sum_{t=1}^{T}\eta_t\Big)^{-1}\sum_{t=1}^{T}\eta_t\big(\ebb_{S,A}[F(\bw_t)]-F(\bw^*)\big)
   = O\Big(\gamma^{\frac{1+\alpha}{\alpha-1}}\Big)
   +O(\gamma)\Big(T^{1-\frac{2\theta}{1-\alpha}}+n^{-2}T^{\frac{2-2\theta}{1+\alpha}}+n^{-2}T^{2-2\theta}F^{\frac{2\alpha}{1+\alpha}}(\bw^*)\Big)\\
   +O(1/\gamma)\Big(T^{-\frac{2\alpha(1-\theta)}{1+\alpha}}+F^{\frac{2\alpha}{1+\alpha}}(\bw^*)\Big)
  +O(T^{\theta-1})+O(T^{\frac{\alpha\theta-\theta-2\alpha}{1+\alpha}})+O(T^{-\theta}F^{\frac{2\alpha}{1+\alpha}}(\bw^*)).
\end{multline*}
The stated inequality then follows from the convexity of $F$. The proof is complete.
\end{proof}

\subsection{Empirical Risk Minimization with Strongly Convex Objectives\label{sec:erm}}
%Under the $G$-admissibility assumption of loss functions closely related to Assumption \ref{ass:lipschitz}, it was shown that the ERM algorithm is $O(G^2/(n\sigma))$-uniformly stable if $F_S$ is $\sigma$-strongly convex~\citep{bousquet2002stability}. We can derive similar generalization bounds without the admissibility assumption by using the on-average model stability, which replaces the admissibility constant $G$ with the population risk $\ebb_{S}\big[F(A(S))\big]$. In a low-noise setting with a small $\ebb_{S}\big[F(A(S))\big]$, our approach can imply better generalization bounds than those based on the uniform stability.
In this section, we present an optimistic bound for ERM with strongly convex objectives based on the $\ell_2$ on-average model stability.
We consider nonnegative and convex loss functions with H\"older continuous (sub)gradients.
\begin{proposition}\label{prop:erm}
  Assume for any $z$, the function $\bw\mapsto f(\bw;z)$ is nonnegative, convex and $\bw\mapsto \partial f(\bw;z)$ is $(\alpha,L)$-H\"older continuous with $\alpha\in[0,1]$. Let $A$ be the ERM algorithm, i.e., $A(S)=\arg\min_{\bw\in\rbb^d}F_S(\bw)$. If for all $S$, $F_S$ is $\sigma$-strongly convex, then
  \[
  \ebb_S\big[F(A(S))-F_S(A(S))\big]\leq\frac{2c^2_{\alpha,1}}{n\sigma}\ebb_{S}\Big[F^{\frac{2\alpha}{1+\alpha}}(A(S))\Big].
  \]
\end{proposition}
\begin{proof}
  Let $\widetilde{S}$ and $S^{(i)},i=1,\ldots,n$, be constructed as Definition \ref{def:aver-stab}. Due to the $\sigma$-strong convexity of $F_{S^{(i)}}$ and $\partial F_{S^{(i)}}(A(S^{(i)}))=0$ (necessity condition for the optimality of $A(S^{(i)})$), we know
  \[
  F_{S^{(i)}}(A(S))-F_{S^{(i)}}(A(S^{(i)}))\geq 2^{-1}\sigma\big\|A(S)-A(S^{(i)})\big\|_2^2.
  \]
  Taking a summation of the above inequality yields
  \begin{equation}\label{erm-1}
    \frac{1}{n}\sum_{i=1}^n\Big(F_{S^{(i)}}(A(S))-F_{S^{(i)}}(A(S^{(i)}))\Big)\geq \frac{\sigma}{2n}\sum_{i=1}^{n}\big\|A(S)-A(S^{(i)})\big\|_2^2.
  \end{equation}
  According to the definition of $S^{(i)}$, we know
  \begin{align*}
     n\sum_{i=1}^{n}F_{S^{(i)}}(A(S)) & = \sum_{i=1}^{n}\Big(\sum_{j\neq i}f(A(S);z_j)+f(A(S);\tilde{z}_i)\Big) \\
     & = (n-1)\sum_{j=1}^{n}f(A(S);z_j)+\sum_{i=1}^{n}f(A(S);\tilde{z}_i)=(n-1)nF_S(A(S))+nF_{\widetilde{S}}(A(S)).
  \end{align*}
  Taking an expectation and dividing both sides by $n^2$ give ($A(S)$ is independent of $\widetilde{S}$)
  \begin{equation}\label{erm-2}
    \frac{1}{n}\ebb_{S,\widetilde{S}}\Big[\sum_{i=1}^{n}F_{S^{(i)}}(A(S))\Big]=\frac{n-1}{n}\ebb_S\big[F_S(A(S))\big]+\frac{1}{n}\ebb_S\big[F(A(S))\big].
  \end{equation}
  Furthermore, by symmetry we know
  \[
  \frac{1}{n}\ebb_{S,\widetilde{S}}\Big[\sum_{i=1}^{n}F_{S^{(i)}}(A(S^{(i)}))\Big]=\ebb_S\big[F_S(A(S))\big].
  \]
  Plugging the above identity and \eqref{erm-2} back into \eqref{erm-1} gives
  \begin{equation}\label{erm-3}
    \frac{\sigma}{2n}\sum_{i=1}^{n}\ebb_{S,\widetilde{S}}\big[\|A(S^{(i)})-A(S)\|_2^2\big]\leq \frac{1}{n}\ebb_{S,\widetilde{S}}\Big[F(A(S))-F_S(A(S))\Big].
  \end{equation}
  We can now apply Part (c) of Theorem \ref{thm:gen-model-stab} to show the following inequality for all $\gamma>0$ (notice $A$ is a deterministic algorithm)
  \[
  \ebb_S\big[F(A(S))-F_S(A(S))\big]\leq\frac{c^2_{\alpha,1}}{2\gamma}\ebb_{S}\Big[F^{\frac{2\alpha}{1+\alpha}}(A(S))\Big]
    +\frac{\gamma}{n\sigma}\ebb_{S}\Big[F(A(S))-F_S(A(S))\Big].
  \]
  Taking $\gamma=n\sigma/2$, we derive
  \[
  \ebb_S\big[F(A(S))-F_S(A(S))\big]\leq\frac{c^2_{\alpha,1}}{n\sigma}\ebb_{S}\Big[F^{\frac{2\alpha}{1+\alpha}}(A(S))\Big]
    +\frac{1}{2}\ebb_{S}\Big[F(A(S))-F_S(A(S))\Big],
  \]
  from which we can derive the stated inequality. The proof is complete.
\end{proof}

\section{Proofs on Stability with Relaxed Convexity}

\subsection{Stability and generalization errors\label{sec:proof-convex}}
For any convex $g$, we have~\citep{nesterov2013introductory}
\begin{equation}\label{monotonicity}
  \langle\bw-\tilde{\bw},\nabla g(\bw)-\nabla g(\tilde{\bw})\rangle\geq0,\quad\bw,\tilde{\bw}\in\rbb^d.
\end{equation}
\begin{proof}[Proof of Theorem \ref{thm:stab-convex}]
Without loss of generality, we can assume that $S$ and $\widetilde{S}$ differ by the first example, i.e., $z_1\neq\tilde{z}_1$ and $z_i=\tilde{z}_i,i\neq1$.
According to the update rule \eqref{SGD} and \eqref{proj-cont}, we know
\begin{align}
  &\|\bw_{t+1}-\tilde{\bw}_{t+1}\|_2^2 \leq \|\bw_t-\eta_t\partial f(\bw_t;z_{i_t})-\tilde{\bw}_t+\eta_t\partial f(\tilde{\bw}_t;\tilde{z}_{i_t})\|_2^2\notag\\
  & = \|\bw_t-\tilde{\bw}_t\|_2^2 + \eta_t^2\|\partial f(\bw_t;z_{i_t})-\partial f(\tilde{\bw}_t;\tilde{z}_{i_t})\|_2^2
  +2\eta_t\langle\bw_t-\tilde{\bw}_t,\partial f(\tilde{\bw}_t;\tilde{z}_{i_t})-\partial f(\bw_t;z_{i_t})\rangle.\label{stab-1}
\end{align}
We first study the term $\|\partial f(\bw_t;z_{i_t})-\partial f(\tilde{\bw}_t;\tilde{z}_{i_t})\|_2$.
The event $i_t\neq 1$ happens with probability $1-1/n$, and in this case it follows from the smoothness of $f$ that ($z_{i_t}=\tilde{z}_{i_t}$)
\[
\|\partial f(\bw_t;z_{i_t})-\partial f(\tilde{\bw}_t;\tilde{z}_{i_t})\|_2\leq L\|\bw_t-\tilde{\bw}_t\|_2.
\]
The event $i_t=1$ happens with probability $1/n$, and in this case
\[
\|\partial f(\bw_t;z_{i_t})-\partial f(\tilde{\bw}_t;\tilde{z}_{i_t})\|_2\leq \|\partial f(\bw_t;z_{i_t})\|_2+\|\partial f(\tilde{\bw}_t;\tilde{z}_{i_t})\|_2\leq 2G.
\]
Therefore, we get
\begin{equation}\label{stab-2}
\ebb_{i_t}\big[\|\partial f(\bw_t;z_{i_t})-\partial f(\tilde{\bw}_t;\tilde{z}_{i_t})\|_2^2\big]\leq \frac{(n-1)L^2}{n}\|\bw_t-\tilde{\bw}_t\|_2^2+
\frac{4G^2}{n}.
\end{equation}
It is clear
\[
\ebb_{i_t}\big[f(\bw_t;z_{i_t})\big]=F_S(\bw_t)\quad\text{and}\quad\ebb_{i_t}\big[f(\tilde{\bw}_t;\tilde{z}_{i_t})\big]=F_{\widetilde{S}}(\tilde{\bw}_t).
\]
Therefore, by \eqref{monotonicity} we derive
\begin{align}
  \ebb_{i_t}\Big[\langle\bw_t-\tilde{\bw}_t,\partial f(\tilde{\bw}_t;\tilde{z}_{i_t})&-\partial f(\bw_t;z_{i_t})\rangle\Big] = \langle\bw_t-\tilde{\bw}_t,\partial F_{\widetilde{S}}(\tilde{\bw}_t)-\partial F_{S}(\bw_t)\rangle \notag\\
   & = \langle\bw_t-\tilde{\bw}_t,\partial F_{\widetilde{S}}(\tilde{\bw}_t)-\partial F_S(\tilde{\bw}_t)\rangle + \langle\bw_t-\tilde{\bw}_t,\partial F_S(\tilde{\bw}_t)-\partial F_{{S}}({\bw}_t)\rangle \notag\\
   & = \frac{1}{n}\langle\bw_t-\tilde{\bw}_t,\partial f(\tilde{\bw}_t;\tilde{z}_1)-\partial f(\tilde{\bw}_t;z_1)\rangle + \langle\bw_t-\tilde{\bw}_t,\partial F_S(\tilde{\bw}_t)-\partial F_{{S}}({\bw}_t)\rangle\notag\\
   & \leq \frac{2G\|\bw_t-\tilde{\bw}_t\|_2}{n}.\label{stab-3} %\leq \frac{1}{2}\eta_t\|\bw_t-\tilde{\bw}_t\|_2^2+2G^2/(n^2\eta_t).
\end{align}
Plugging \eqref{stab-2} and the above inequality back into \eqref{stab-1}, we derive
\[
\ebb_{i_t}\big[\|\bw_{t+1}-\tilde{\bw}_{t+1}\|_2^2\big]\leq \|\bw_t-\tilde{\bw}_t\|_2^2\\
+ \frac{4G\eta_t\|\bw_t-\tilde{\bw}_t\|_2}{n}+
\eta_t^2\Big(\frac{(n-1)L^2}{n}\|\bw_t-\tilde{\bw}_t\|_2^2+\frac{4G^2}{n}\Big)
%\frac{2}{n}\big[\|\partial f(\bw_t;z_{i_t})\|_2^2+\|\partial f(\tilde{\bw}_t;\tilde{z}_{i_t})\|_2^2\big]
\]
and therefore
\begin{equation}\label{stab-4}
\ebb_A\big[\|\bw_{t+1}-\tilde{\bw}_{t+1}\|_2^2\big]\leq (1+L^2\eta_t^2)\ebb_A\big[\|\bw_t-\tilde{\bw}_t\|_2^2\big]+4G\Big(\frac{\eta_t\ebb_A[\|\bw_t-\tilde{\bw}_t\|_2]}{n}+\frac{G\eta_t^2}{n}\Big).
\end{equation}
By the above recurrence relationship and $\bw_1=\tilde{\bw}_1$, we derive
\begin{align*}
  \ebb_A\big[\|\bw_{t+1}-\tilde{\bw}_{t+1}\|_2^2\big] & \leq 4G\sum_{j=1}^{t}\prod_{\tilde{j}=j+1}^{t}\Big(1+L^2\eta_{\tilde{j}}^2\Big)\Big(\frac{\eta_j\ebb_A[\|\bw_j-\tilde{\bw}_j\|_2]}{n}+\frac{G\eta_j^2}{n}\Big) \\
   & \leq 4G\prod_{\tilde{j}=1}^{t}\Big(1+L^2\eta_{\tilde{j}}^2\Big)\sum_{j=1}^{t}\Big(\frac{\eta_j\max_{1\leq\tilde{j}\leq t}\ebb_A[\|\bw_{\tilde{j}}-\tilde{\bw}_{\tilde{j}}\|_2]}{n}+\frac{G\eta_j^2}{n}\Big).
\end{align*}
Since the above inequality holds for all $t\in\nbb$ and the right-hand side is an increasing function of $t$, we get
\[
\max_{1\leq\tilde{j}\leq t+1}\ebb_A\big[\|\bw_{\tilde{j}}-\tilde{\bw}_{\tilde{j}}\|_2^2\big]
\leq 4GC_t\sum_{j=1}^{t}\Big(\frac{\eta_j\max_{1\leq\tilde{j}\leq t+1}\ebb_A[\|\bw_{\tilde{j}}-\tilde{\bw}_{\tilde{j}}\|_2]}{n}+\frac{G\eta_j^2}{n}\Big).
\]
It then follows that (note $\ebb_A[\|\bw_{\tilde{j}}-\tilde{\bw}_{\tilde{j}}\|_2]\leq \big(\ebb_A\big[\|\bw_{\tilde{j}}-\tilde{\bw}_{\tilde{j}}\|_2^2\big]\big)^{\frac{1}{2}}$)
\[
\max_{1\leq\tilde{j}\leq t+1}\ebb_A\big[\|\bw_{\tilde{j}}-\tilde{\bw}_{\tilde{j}}\|_2^2\big] \leq
4GC_t\sum_{j=1}^{t}\frac{\eta_j}{n}\max_{1\leq\tilde{j}\leq t+1}\Big(\ebb_A\big[\|\bw_{\tilde{j}}-\tilde{\bw}_{\tilde{j}}\|_2^2\big]\Big)^{\frac{1}{2}}+4G^2C_t\sum_{j=1}^{t}\frac{\eta_j^2}{n}.
\]
Solving the above quadratic function of $\max_{1\leq\tilde{j}\leq t+1}\Big(\ebb_A\big[\|\bw_{\tilde{j}}-\tilde{\bw}_{\tilde{j}}\|_2^2\big]\Big)^{\frac{1}{2}}$ then shows
\[
\max_{1\leq\tilde{j}\leq t+1}\Big(\ebb_A\big[\|\bw_{\tilde{j}+1}-\tilde{\bw}_{\tilde{j}+1}\|_2^2\big]\Big)^{\frac{1}{2}}\leq
4GC_t\sum_{j=1}^{t}\frac{\eta_j}{n}+2G\Big(C_t\sum_{j=1}^{t}\frac{\eta_j^2}{n}\Big)^{\frac{1}{2}}.
\]
The proof is complete.
\end{proof}

To prove Theorem \ref{thm:error-convex}, we require a basic result on series.
\begin{lemma}\label{lem:series}We have the following elementary inequalities.
\begin{enumerate}[(a)]
  \item If $\theta\in(0,1)$, then $(t^{1-\theta}-1)/(1-\theta)\leq \sum_{k=1}^{t}k^{-\theta}\leq t^{1-\theta}/(1-\theta)$;
  %\item If $\theta=1$, then $\sum_{k=1}^{t}k^{-\theta}\leq\log(et)$;
  \item If $\theta>1$, then $\sum_{k=1}^{t}k^{-\theta}\leq\frac{\theta}{\theta-1}$.
\end{enumerate}
\end{lemma}
We denote by $\epsilon_{\stab}(A,n)$ the infimum over all $\epsilon$ for which \eqref{unif-stab} holds, and omit the tuple $(A,n)$ when it is clear from the context.
\begin{proof}[Proof of Theorem \ref{thm:error-convex}]
  For the step sizes considered in both Part (a) and Part (b), one can check that $\sum_{t=1}^{T}\eta_t^2$ can be upper bounded by a constant independent of $T$. Therefore, $C_t<C$ for all $t=1,\ldots,T$ and a universal constant $C$.
  We can apply Lemma \ref{lem:computation-ave} (Part (a)) on optimization errors to get
  \begin{equation}\label{computation-holder}
  \ebb_A[F_S(\bw_T^{(1)})]-F_S(\bw^*)= O\Big(\frac{\sum_{t=1}^{T}\eta_t^2+\|\bw^*\|_2^2}{\sum_{t=1}^{T}\eta_t}\Big).
  \end{equation}
  By the convexity of norm, we know
  \[
  \ebb_A\big[\|\bw_T^{(1)}-\tilde{\bw}_T^{(1)}\|_2\big]\leq \big(\sum_{t=1}^{T}\eta_t\big)^{-1}\sum_{t=1}^{T}\eta_t\ebb_A[\|\bw_t-\tilde{\bw}_t\|_2]=
  O\Big(\sum_{t=1}^{T}\frac{\eta_t}{n}+n^{-\frac{1}{2}}\big(\sum_{t=1}^{T}\eta_t^2\big)^{\frac{1}{2}}\Big),
  \]
  where we have applied Theorem \ref{thm:stab-convex} (the upper bound in Theorem \ref{thm:stab-convex} is an increasing function of $t$).
  It then follows from the Lipschitz continuity that $\epsilon_{\stab}=O\Big(\sum_{t=1}^{T}\frac{\eta_t}{n}+n^{-\frac{1}{2}}\big(\sum_{t=1}^{T}\eta_t^2\big)^{\frac{1}{2}}\Big)$.
  This together with the error decomposition \eqref{decomposition}, Lemma \ref{lem:gen-stab} and the optimization error bound \eqref{computation-holder} shows
  \begin{equation}\label{gen-convex-1}
  \ebb_{S,A}[F(\bw_T^{(1)})]-F(\bw^*)=O\Big(\sum_{t=1}^{T}\frac{\eta_t}{n}+n^{-\frac{1}{2}}\big(\sum_{t=1}^{T}\eta_t^2\big)^{\frac{1}{2}}\Big)+
  O\Big(\frac{\sum_{t=1}^{T}\eta_t^2+\|\bw^*\|_2^2}{\sum_{t=1}^{T}\eta_t}\Big).
  \end{equation}
  For the step sizes $\eta_t=\eta_1t^{-\theta}$ with $\theta\in(1/2,1)$, we can apply Lemma \ref{lem:series} to show
  \begin{align*}
  \ebb_{S,A}[F(\bw_T^{(1)})]-F(\bw^*)&=O\Big(\sum_{t=1}^{T}\frac{t^{-\theta}}{n}+n^{-\frac{1}{2}}\big(\sum_{t=1}^{T}t^{-2\theta}\big)^{\frac{1}{2}}\Big)+
  O\Big(\frac{\sum_{t=1}^{T}t^{-2\theta}+\|\bw^*\|_2^2}{\sum_{t=1}^{T}t^{-\theta}}\Big)\\
  & = O\Big(n^{-1}T^{1-\theta}+n^{-\frac{1}{2}}+T^{\theta-1}\Big).
  \end{align*}
  This proves the first part.

  Part (b) follows by plugging \eqref{cor-gen-lipschitz-1} into \eqref{gen-convex-1}.
  The proof is complete.
  %where in the last step we have used $T\asymp n^{\frac{1}{2-2\theta}}$.
\end{proof}

\section{Proofs on Stability with Relaxed Strong Convexity\label{sec:proof-strong}}
\begin{proof}[Proof of Theorem \ref{thm:stable-strong}]
Due to the $\sigma_S$-strong convexity of $F_S$, we can analyze analogously to \eqref{stab-3} to derive
\[
\ebb_{i_t}\Big[\langle\bw_t-\tilde{\bw}_t,\partial f(\tilde{\bw}_t;z_{i_t})-\partial f(\bw_t;z_{i_t})\rangle\Big]
\leq \frac{2G\|\bw_t-\tilde{\bw}_t\|_2}{n}-\sigma_S\|\bw_t-\tilde{\bw}_t\|_2^2.
\]
Therefore, analogous to the derivation of \eqref{stab-4} we can derive
\begin{align*}
\ebb_A\big[\|\bw_{t+1}-\tilde{\bw}_{t+1}\|_2^2\big] &\leq (1+L^2\eta_t^2-2\sigma_S\eta_t)\ebb_A\big[\|\bw_t-\tilde{\bw}_t\|_2^2\big]+4G\Big(\frac{G\eta_t^2}{n}+\frac{\eta_t\ebb_A[\|\bw_t-\tilde{\bw}_t\|_2]}{n}\Big)\\
&\leq (1+L^2\eta_t^2-\frac{3}{2}\sigma_S\eta_t)\ebb_A\big[\|\bw_t-\tilde{\bw}_t\|_2^2\big]+\frac{4G^2\eta_t^2}{n}+\frac{8G^2\eta_t}{n^2\sigma_S},
\end{align*}
where we have used
\[
\frac{4G}{n}\ebb_A[\|\bw_t-\tilde{\bw}_t\|_2]\leq \frac{8G^2}{n^2\sigma_S}+\frac{\sigma_S\ebb_A[\|\bw_t-\tilde{\bw}_t\|_2^2]}{2}.
\]
We find $t_0\geq4L^2/\sigma_S^2$. Then $\eta_t\leq\sigma_S/(2L^2)$ and it follows that
\begin{align*}
\ebb_A\big[\|\bw_{t+1}-\tilde{\bw}_{t+1}\|_2^2\big] &\leq (1-\sigma_S\eta_t)\ebb_A\big[\|\bw_t-\tilde{\bw}_t\|_2^2\big]+\frac{4G^2\eta_t^2}{n}+\frac{8G^2\eta_t}{n^2\sigma_S}\\
& = \Big(1-\frac{2}{t+t_0}\Big)\ebb_A\big[\|\bw_t-\tilde{\bw}_t\|_2^2\big] + \frac{4G^2}{n}\Big(\eta_t^2+\frac{2\eta_t}{n\sigma_S}\Big).
\end{align*}
Multiplying both sides by $(t+t_0)(t+t_0-1)$ yields
\[
(t+t_0)(t+t_0-1)\ebb_A\big[\|\bw_{t+1}-\tilde{\bw}_{t+1}\|_2^2\big]\leq (t+t_0-1)(t+t_0-2)\ebb_A\big[\|\bw_t-\tilde{\bw}_t\|_2^2\big]+\frac{8G^2(t+t_0-1)}{n\sigma_S}\Big(\eta_t+\frac{2}{n\sigma_S}\Big).
\]
Taking a summation of the above inequality and using $\bw_1=\tilde{\bw}_1$ then give
\begin{align*}
  (t+t_0)(t+t_0-1)\ebb_A\big[\|\bw_{t+1}-\tilde{\bw}_{t+1}\|_2^2\big] & \leq \frac{8G^2}{n\sigma_S}\sum_{j=1}^{t}(j+t_0-1)\Big(\eta_j+\frac{2}{n\sigma_S}\Big) \\
   & = \frac{8G^2}{n\sigma_S}\Big(\sum_{j=1}^{t}(j+t_0-1)\eta_j+\frac{2}{n\sigma_S}\sum_{j=1}^{t}(j+t_0-1)\Big) \\
   & \leq \frac{8G^2}{n\sigma_S}\Big(\frac{2t}{\sigma_S}+\frac{t(t+2t_0-1)}{n\sigma_S}\Big).% \\
   %& = \frac{8G^2}{n\sigma_S^2}\Big(2t+\frac{t(t+2t_0-1)}{n}\Big).
\end{align*}
It then follows
\[
\ebb_A\big[\|\bw_{t+1}-\tilde{\bw}_{t+1}\|_2^2\big]\leq \frac{16G^2}{n\sigma_S^2}\Big(\frac{1}{t+t_0}+\frac{1}{n}\Big).
\]
The stated bound then follows from the elementary inequality $\sqrt{a+b}\leq\sqrt{a}+\sqrt{b}$ for $a,b\geq0$. The proof is complete.
\end{proof}

\begin{proof}[Proof of Theorem \ref{thm:error-strong}]
  By the convexity of norm, we know
  \begin{align*}
  \ebb_A\big[\|\bw_T^{(2)}-\tilde{\bw}_T^{(2)}\|_2\big] &\leq \big(\sum_{t=1}^{T}(t+t_0-1)\big)^{-1}\sum_{t=1}^{T}(t+t_0-1)\ebb_A[\|\bw_t-\tilde{\bw}_t\|_2]\\
  & \leq \frac{4G}{\sigma_S}\big(\sum_{t=1}^{T}(t+t_0-1)\big)^{-1}\sum_{t=1}^{T}(t+t_0-1)\Big(\frac{1}{\sqrt{n(t+t_0)}}+\frac{1}{n}\Big)\\
  & = O(\sigma_{S}^{-1}\big((nT)^{-\frac{1}{2}}+n^{-1}\big)),
  \end{align*}
  where we have used Lemma \ref{lem:series} in the last step.
  Since the above bound holds for all $S,\widetilde{S}$ differing by a single example, it follows that $\ell_1$ on-average model stability is bounded by $O(\ebb_S[\sigma_{S}^{-1}]\big((nT)^{-\frac{1}{2}}+n^{-1}\big))$.
  By Part (b) of Lemma \ref{lem:computation-ave} we know
  \[
    \ebb_A[F_S(\bw_T^{(2)})]-F_S(\bw^*)=O\big(1/(T\sigma_S)+\|\bw^*\|_2^2/T^2\big).
  \]
  It then follows from \eqref{decomposition} and Part (a) of Theorem \ref{thm:gen-model-stab} that
  \[
  \ebb_{S,A}[F(\bw_T^{(2)})]-F(\bw^*)=O(\ebb_S\big[\sigma_{S}^{-1}\big((nT)^{-\frac{1}{2}}+n^{-1}\big)\big])+O\big(\ebb_S\big[1/(T\sigma_S)\big]+1/T^2\big).
  \]
  The stated bound holds since $T\asymp n$.
  The proof is complete.
\end{proof}

\begin{proposition}\label{prop:span}
  Let $S=\{z_1,\ldots,z_n\}$ and $C_S=\frac{1}{n}\sum_{i=1}^{n}x_ix_i^\top$. Then the range of $C_S$ is the linear span of $\{x_1,\ldots,x_n\}$.
\end{proposition}
\begin{proof}
  It suffices to show that the kernel of $C_S$ is the orthogonal complement of $V=\text{span}\{x_1,\ldots,x_n\}$ (we denote $\text{span}\{x_1,\ldots,x_n\}$ the linear span of $x_1,\ldots,x_n$). Indeed, for any $x$ in the kernel of $C_S$, we know $C_Sx=0$ and therefore $x^\top C_Sx=\frac{1}{n}\sum_{i=1}^{n}(x_i^\top x)^2=0$, from which we know that $x$ must be orthogonal to $V$. Furthermore, for any $x$ orthogonal to $V$, it is clear that $C_Sx=0$, i.e., $x$ belongs to the kernel of $C_S$. The proof is complete.
\end{proof}

\section{Extensions\label{sec:extension}}
In this section, we present some extensions of our analyses. We consider three extensions: extension to stochastic proximal gradient descent, extension to high probability analysis and extension to SGD without replacement.

\subsection{Stochastic proximal gradient descent}
Our discussions can be directly extended to study the performance of stochastic proximal gradient descent (SPGD).  Let $r:\rbb^d\to\rbb^+$ be a convex regularizer. SPGD updates the models by
\[
\bw_{t+1}=\prox_{\eta_tr}\big(\bw_t-\eta_t\partial f(\bw_t;z_{i_t})\big),
\]
where $\prox_g(\bw)=\arg\min_{\tilde{\bw}\in\rbb^d}\big[g(\tilde{\bw})+\frac{1}{2}\|\bw-\tilde{\bw}\|_2^2\big]$ is the proximal operator. SPGD has found wide applications in solving optimization problems with a composite structure~\citep{parikh2014proximal}. It recovers the projected SGD as a specific case by taking an appropriate $r$. Our stability bounds for SGD can be trivially extend to SPGD due to the non-expansiveness of proximal operators: $\|\prox_g(\bw)-\prox_g(\tilde{\bw})\|_2\leq\|\bw-\tilde{\bw}\|_2,\forall \bw,\tilde{\bw}$ if $g$ is convex.

\subsection{Stability bounds with high probabilities}

We can also extend our stability bounds stated in expectation to high-probability bounds, which would be helpful to understand the fluctuation of SGD w.r.t. different realization of random indices.
\begin{proposition}\label{prop:high-probability}
  Let Assumption \ref{ass:lipschitz} hold. Assume for all $z\in\zcal$, the function $\bw\mapsto f(\bw;z)$ is convex and $\bw\mapsto \partial f(\bw;z)$ is $(\alpha,L)$-H\"older continuous with $\alpha\in[0,1]$.
  Let $S=\{z_1,\ldots,z_n\}$ and $\widetilde{S}=\{\tilde{z}_1,\ldots,\tilde{z}_n\}$ be two sets of training examples that differ by a single example.
  Let $\{\bw_t\}_t$ and $\{\tilde{\bw}_t\}_t$ be produced by \eqref{SGD} based on $S$ and $\widetilde{S}$, respectively, and $\delta\in(0,1)$. If we take step size $\eta_j=ct^{-\theta}$ for $j=1,\ldots,t$ and $c>0$, then with probability at least $1-\delta$
  \[
  \|\bw_{t+1}\!-\!\tilde{\bw}_{t+1}\|_2\!=\!O\Big(t^{1-\frac{\theta}{1-\alpha}}+n^{-1}t^{1-\theta}\Big(1\!+\!\sqrt{nt^{-1}\log(1/\delta)}\Big)\Big).
  \]
\end{proposition}
High-probability generalization bounds can be derived by combining the above stability bounds and the recent result on relating generalization and stability in a high-probability analysis~\citep{feldman2019high,bousquet2019sharper}.

To prove Proposition \ref{prop:high-probability}, we need to introduce a special concentration inequality called Chernoff's fs bound for a summation of independent Bernoulli random variables~\citep{boucheron2013concentration}.
\begin{lemma}[Chernoff's Bound\label{lem:chernoff}]
  Let $X_1,\ldots,X_t$ be independent random variables taking values in $\{0,1\}$. Let $X=\sum_{j=1}^{t}X_j$ and $\mu=\ebb[X]$. Then for any $\tilde{\delta}\in(0,1)$ with probability at least $1-\exp\big(-\mu\tilde{\delta}^2/3\big)$ we have $X\leq (1+\tilde{\delta})\mu$.
\end{lemma}
\begin{proof}[Proof of Proposition \ref{prop:high-probability}]
Without loss of generality, we can assume that $S$ and $\widetilde{S}$ differ by the first example, i.e., $z_1\neq\tilde{z}_1$ and $z_i=\tilde{z}_i$ for $i\neq1$.
If $i_t\neq 1$, we can apply Lemma \ref{lem:holder-expansive} and \eqref{proj-cont} to derive
\begin{align*}
    \|\bw_{t+1}-\tilde{\bw}_{t+1}\|_2 & \leq \|\bw_t-\eta_t\partial f(\bw_t;z_{i_t})-\tilde{\bw}_t+\eta_t\partial f(\tilde{\bw}_t;z_{i_t})\|_2 \\
     & \leq \|\bw_t-\tilde{\bw}_t\|_2 + c_{\alpha,3}\eta_t^{\frac{1}{1-\alpha}}.
  \end{align*}
If $i_t=1$, we know
  \begin{align*}
    \|\bw_{t+1}-\tilde{\bw}_{t+1}\|_2 & \leq \|\bw_t-\eta_t\partial f(\bw_t;z_1)-\tilde{\bw}_t+\eta_t\partial f(\tilde{\bw}_t;\tilde{z}_1)\|_2 \\
     & \leq \|\bw_t-\tilde{\bw}_t\|_2 + 2\eta_tG.
\end{align*}
Combining the above two cases together, we derive
\[
\|\bw_{t+1}-\tilde{\bw}_{t+1}\|_2
      \leq \|\bw_t-\tilde{\bw}_t\|_2 + c_{\alpha,3}\eta_t^{\frac{1}{1-\alpha}}+2\eta_tG\ibb_{[i_t=1]}.
\]
Taking a summation of the above inequality then yields
\[
\|\bw_{t+1}-\tilde{\bw}_{t+1}\|_2 \leq
c_{\alpha,3}\sum_{j=1}^{t}\eta_j^{\frac{1}{1-\alpha}}+2G\sum_{j=1}^{t}\eta_j\ibb_{[i_j=1]}.
\]
Applying Lemma \ref{lem:chernoff} with $X_j=\ibb_{[i_j=1]}$ and $\mu=t/n$ (note $\ebb_A[X_j]=1/n$), with probability $1-\delta$ there holds
\[
\sum_{j=1}^{t}\ibb_{[i_j=1]}\leq \frac{t}{n}\big(1+\sqrt{3nt^{-1}\log(1/\delta)}\big).
\]
Therefore, for the step size $\eta_j=ct^{-\theta},j=1,\ldots,t$, we know
\[
\|\bw_{t+1}-\tilde{\bw}_{t+1}\|_2\leq c_{\alpha,3}c^{\frac{1}{1-\alpha}}t^{1-\frac{\theta}{1-\alpha}}
+2Gcn^{-1}
\big(1+\sqrt{3nt^{-1}\log(1/\delta)}\big)t^{1-\theta}.
\]
The proof is complete.
\end{proof}

\subsection{SGD without replacement}

Our stability bounds can be further extended to SGD without replacement. In this case, we run SGD in epochs. For the $k$-th epoch, we start with a model $\bw^k_1\in\rbb^d$, and draw an index sequence $(i_1^k,\ldots,i_n^k)$ from the uniform distribution over all permutations of $\{1,\ldots,n\}$. Then we update the model by
\begin{equation}\label{SGD-replacement}
  \bw^k_{t+1}=\bw^k_t-\eta^k_t\partial f(\bw^k_t;z_{i_t^k}),\quad t=1,\ldots,n,
\end{equation}
where $\{\eta_t^k\}$ is the step size sequence.
We set $\bw^{k+1}_1=\bw^k_{n+1}$, i.e., each epoch starts with the last iterate of the previous epoch.
The following proposition establishes stability bounds for SGD without replacement when applied to loss functions with H\"older continuous (sub)gradients.
\begin{proposition}\label{prop:stab-replacement}
  Suppose assumptions of Proposition \ref{prop:high-probability} hold.
  Let $\{\bw_t\}_t$ and $\{\tilde{\bw}_t\}_t$ be produced by \eqref{SGD-replacement} based on $S$ and $\widetilde{S}$, respectively.  Then
  \[
    \ebb_A[\|\bw_{1}^{K+1}-\tilde{\bw}_1^{K+1}\|_2]\leq\frac{2G}{n}\sum_{k=1}^{K}\sum_{t=1}^{n}\eta^{k}_t+c_{\alpha,3}\sum_{k=1}^{K}\sum_{t=1}^{n}(\eta_t^k)^{\frac{1}{1-\alpha}}.
  \]
\end{proposition}
\begin{proof}%[Proof of Proposition \ref{prop:stab-replacement}]
  Without loss of generality, we can assume that $S$ and $\widetilde{S}$ differ by the first example, i.e., $z_1\neq\tilde{z}_1$ and $z_i=\tilde{z}_i$ for $i\neq1$.
  Analogous to the proof of Proposition \ref{prop:high-probability}, we derive the following inequality for all $k\in\nbb$ and $t=1,\ldots,n$
  \[
  \|\bw^k_{t+1}-\tilde{\bw}^k_{t+1}\|_2
     \leq \|\bw_t^k-\tilde{\bw}_t^k\|_2 + c_{\alpha,3}(\eta_t^k)^{\frac{1}{1-\alpha}}\ibb_{[i_t^k\neq1]}+2\eta^k_tG\ibb_{[i_t^k=1]}.
  \]
  Taking a summation of the above inequality from $t=1$ to $n$ gives
  \[
  \|\bw^k_{n+1}-\tilde{\bw}^k_{n+1}\|_2
      \leq \|\bw_1^k-\tilde{\bw}_1^k\|_2 + c_{\alpha,3}\sum_{t=1}^{n}(\eta_t^k)^{\frac{1}{1-\alpha}}\ibb_{[i_t^k\neq1]}+2G\sum_{t=1}^{n}\eta^k_t\ibb_{[i_t^k=1]}.
  \]
  Let $i^k$ be the unique $t\in\{1,\ldots,n\}$ such that $i_t^k=1$. Since $\bw^{k+1}_1=\bw^k_{n+1}$, we derive
  \[
  \|\bw^{k+1}_{1}-\tilde{\bw}^{k+1}_{1}\|_2
      \leq \|\bw_1^k-\tilde{\bw}_1^k\|_2 + c_{\alpha,3}\sum_{t=1}^{n}(\eta_t^k)^{\frac{1}{1-\alpha}}+2G\eta^k_{i^k}.
  \]
  Since we draw $(i^k_1,\ldots,i^k_n)$ from the uniform distribution of all permutations, $i^k$ takes an equal probability to each $1,\ldots,n$. Therefore, we can take expectations over $A$ to derive
  \[
  \ebb_A\big[\|\bw^{k+1}_{1}-\tilde{\bw}^{k+1}_{1}\|_2\big]
      \leq \ebb_A\big[\|\bw_1^k-\tilde{\bw}_1^k\|_2\big] + c_{\alpha,3}\sum_{t=1}^{n}(\eta_t^k)^{\frac{1}{1-\alpha}}+\frac{2G\sum_{t=1}^{n}\eta_t^k}{n}.
  \]
  We can take a summation of the above inequality from $k=1$ to $K$ to derive the stated bound. The proof is complete.
\end{proof}

\setlength{\bibsep}{0.06cm}
\bibliographystyle{icml2020}%abbrvnat

%\bibliography{learning}

\end{document}